%% file: main.tex
\documentclass[hidelinks,11pt]{article}
\usepackage{fullpage}
\usepackage[utf8]{inputenc}
\usepackage{amsmath}
\usepackage{hyperref}       
\usepackage{amsfonts}       
\usepackage{amsmath, amssymb, graphicx, url} 
\usepackage{amsthm} 
\usepackage{nicefrac}       
\usepackage{microtype}      
\usepackage{color}
\usepackage{algorithm}
\usepackage{algpseudocode}
\usepackage{mathtools}
\usepackage{epstopdf}
\usepackage{bm}
\usepackage{tabulary}
\usepackage[mathscr]{euscript}
\usepackage{mathrsfs}
 \usepackage{natbib}
\allowdisplaybreaks
\usepackage{tikz}
\usetikzlibrary{shapes,arrows}
\usetikzlibrary{arrows.meta}
\usetikzlibrary{positioning}
\usetikzlibrary{calc}
\usetikzlibrary{backgrounds}
\tikzset{block/.style = {draw, text=black,rectangle, rounded corners,
		minimum height=2.5em, minimum width=5em}
}
 \newtheorem{theorem}{\bf Theorem}
 \newtheorem{lemma}{\bf Lemma}
 \newtheorem{proposition}{\bf Proposition}
 
\newtheorem{remark}{\bf Remark}

\newtheorem{problem}{\bf Problem}
\newtheorem{assumption}{\bf Assumption}
\newtheorem{definition}{\bf Definition}

\newcommand{\paren}[1]{\ensuremath{\left( #1\right)}}

\newcommand{\clint}[1]{\ensuremath{\left[ #1\right]}}
\newcommand{\set}[1]{\ensuremath{\left\{ #1\right\}}}
\newcommand{\matr}[1]{\ensuremath{\clint{\begin{array} #1 \end{array}}}}
\newcommand{\norm}[1]{\ensuremath{\left\| #1\right\|}}
\newcommand{\snorm}[1]{\ensuremath{\| #1\|}}
\newcommand{\abs}[1]{\ensuremath{\left| #1\right|}}
\newcommand{\mbf}[1]{\ensuremath{\bm{#1}}}

\newcommand{\K}{\ensuremath{\mathcal{K}}}
\newcommand{\poly}{\ensuremath{\mathrm{poly}}}

\newcommand{\rank}{\ensuremath{\mathrm{rank}}}
\newcommand{\N}{\ensuremath{\mathcal{N}}}
\newcommand{\R}{\ensuremath{\mathbb{R}}}

\renewcommand{\P}{\ensuremath{\mathbb{P}}}

\newcommand{\F}{\ensuremath{\mathcal{F}}}

\newcommand{\E}{\ensuremath{\mathbb{E}}}

\newcommand{\C}{\ensuremath{\mathcal{C}}}
\newcommand{\CC}{\ensuremath{\mathscr{C}}}

\newcommand{\VEC}{\ensuremath{\mathrm{vec}}}
\newcommand{\AUX}{\ensuremath{\mathrm{AUX}}}

\DeclareMathOperator{\Tr}{\mathrm{tr}}

\hypersetup{colorlinks,
            linkcolor=blue,
            citecolor=blue,
            urlcolor=magenta,
            linktocpage,
            plainpages=false}
\title{Learning to Control Linear Systems can be Hard}
\author{Anastasios~Tsiamis$^1$, Ingvar~Ziemann$^2$, Manfred Morari$^3$, Nikolai Matni$^3$, and George~J.~Pappas$^3$ 
	\date{
	$^1$ Automatic Control Laboratory, ETH Zurich, email: atsiamis@control.ee.ethz.ch\\
	$^2$ School of Electrical Engineering and Computer Science, KTH Royal Institute of Technology, email: ziemann@kth.se\\
	$^3$ Department of Electrical and Systems Engineering, University of Pennsylvania, email: \{morari,nmatni,pappasg\}@seas.upenn.edu}
		}
\begin{document}

\maketitle

\begin{abstract}%
In this paper, we study the statistical difficulty of learning to control linear systems. We focus on two standard benchmarks, the sample complexity of stabilization, and the regret of the online learning of the Linear Quadratic Regulator (LQR). Prior results state that the statistical difficulty for both benchmarks scales polynomially with the system state dimension up to system-theoretic quantities. However, this does not reveal the whole picture. By utilizing minimax lower bounds for both benchmarks, we prove that there exist non-trivial classes of systems for which learning complexity scales dramatically, i.e. exponentially, with the system dimension. This situation arises in the case of underactuated systems, i.e. systems with fewer inputs than states. Such systems are structurally difficult to control and their
system theoretic quantities can scale exponentially with the system dimension dominating learning complexity. Under some additional structural assumptions (bounding systems away from uncontrollability), we provide qualitatively matching upper bounds. We prove that learning complexity can be at most exponential with the controllability index of the system, that is the degree of underactuation.
\end{abstract}

\input{introduction}
\input{formulation}
\input{controllability}

\input{stabilization}

\input{onlineLQR}

\input{conclusion}
\section*{Acknowledgment}
This work was supported by the AFOSR Assured Autonomy grant.
\bibliographystyle{plainnat}
\bibliography{literature}
\counterwithin{theorem}{section}
\counterwithin{equation}{section}
\input{appendix}

\end{document}

%% file: introduction.tex
\section{Introduction}\label{sec:introduction}
In stochastic linear control, the goal is to design a controller for a system of the form
\begin{equation}\label{CTRL_eq:system}
S:\qquad    x_{k+1}=Ax_k+Bu_k+Hw_k,
\end{equation}
where $x_k\in\R^{n}$ is the system internal state, $u_k\in \R^{p}$ is some exogenous input, and $w_k\in \R^r$ is some random disturbance sequence. Matrices $A,\,B,\,H$ determine the evolution of the state, based on the previous state, control input, and disturbance respectively. Control theory has a long history of studying how to design controllers for system~\eqref{CTRL_eq:system} when its model is \emph{known}~\citep{bertsekas2017dynamic}.
However, in reality system~\eqref{CTRL_eq:system} might be \emph{unknown} and we might not have access to its model. In this case, we have to learn how to control~\eqref{CTRL_eq:system} based on data. 

Controlling unknown dynamical systems has also been studied from the perspective of Reinforcement Learning (RL). Although the setting of tabular RL is relatively well-understood~\citep{jaksch2010near}, it has been challenging to analyze the continuous setting, where the state and/or action spaces are infinite~\citep{ortner2012online,kakade2020information}.
Recently, there has been renewed interest in learning to control linear systems. Indeed, linear systems are simple enough to allow for an in-depth theoretical analysis, yet exhibit sufficiently rich behavior so that we can draw conclusions about continuous control of more general system classes~\citep{recht2018tour}. In this paper we focus on the following two problems.

\textbf{Regret of online LQR.} A fundamental benchmark for continuous control is the Linear Quadratic Regulator (LQR) problem, where the goal is to compute a policy
\footnote{A policy decides the current control input $u_t$ based on past state-input values
--see Section~\ref{CTRL_sec:formulation} for details.
} 
$\pi$ that minimizes
\begin{equation}\label{CTRL_eq:LQR_objective}
    J^*(S)\triangleq \min_{\pi}\lim_{T\rightarrow \infty} \frac{1}{T} \E_{S,\pi} \clint{\sum^{T-1}_{t=0}(x'_tQx_t+u'_tRu_t)+x'_TQ_Tx_T},
\end{equation}
where $Q\in\R^{n\times n}$, $R\in\R^{p\times p}$ are the state and input penalties respectively; these penalties control the tradeoff between state regulation and control effort. 
 When model~\eqref{CTRL_eq:system} is known, LQR enjoys a closed-form solution; the optimal policy is a linear feedback law  $\pi_{\star,t}(x_t)=K_{\star}x_{t}$, where the control gain $K_{\star}$ is given by solving the celebrated Algebraic Riccati Equation (ARE)~\eqref{CTRL_eq:LQR_DARE}.
 If model~\eqref{CTRL_eq:system} is unknown, we have to learn the optimal policy from data. In the online learning setting, the goal of the learner
 is to find a policy that adapts online and competes with the optimal LQR policy that has access to the true model. The suboptimality of the online learning policy at time $T$ is captured by the \emph{regret}
\begin{equation}\label{CTRL_eq:regret}
    R_T(S)\triangleq \sum_{t=0}^{T-1}(x'_tQx_t+u'_tRu_t)+x'_TQ_Tx_T-TJ^*(S).
\end{equation}
The learning task is to find a policy with as small regret as possible.

\textbf{Sample Complexity of Stabilization} Another important benchmark is the problem of stabilization from data. The goal is to learn a linear gain $K\in\R^{m\times n}$ such that the closed-loop system $A+BK$ is stable, i.e., such that its spectral radius $\rho(A+BK)$ is less than one. 
Many algorithms for online LQR require the existence of such a stabilizing gain to initialize the online learning policy~\citep{simchowitz2020naive,jedra2021minimal}. Furthermore, stabilization is a problem of independent interest~\citep{faradonbeh2018stabilization}.
In this setting, the learner designs an exploration policy $\pi$ and an algorithm that uses batch state-input data $x_0,\dots,x_N,u_0,\dots,u_{N-1}$ to output a control gain $\hat{K}_N$, at the end of the exploration phase. Here we focus on \emph{sample complexity}, i.e., the minimum number of samples $N$ required to find a stabilizing gain.

Since the seminal papers by~\cite{abbasi2011regret} and~\cite{dean2017sample} both LQR and stabilization have been studied extensively in the literature -- see Section~\ref{CTLR_sec:related_work}. Current state-of-the-art results state that the regret of online LQR and the sample complexity of stabilization scale at most polynomially with system dimension $n$
\begin{equation}\label{CTRL_eq:existing_upper_bounds}
R_T(S)\lesssim C^{\mathrm{sys}}_{1}\poly(n)\sqrt{T},\quad N\lesssim C^{\mathrm{sys}}_{\mathrm{2}}\poly(n),
\end{equation}
where $C^{\mathrm{sys}}_{1},\,C^{\mathrm{sys}}_{2}$ are system specific constants that depend on several control theoretic quantities of system~\eqref{CTRL_eq:system}. 
However, the above statements might not reveal the whole picture. 

In fact, system theoretic parameters $C^{\mathrm{sys}}_{1},\,C^{\mathrm{sys}}_{2}$ can actually hide dimensional dependence on $n$. This dependence has been overlooked in prior work. As we show in this paper, there exist non-trivial classes of linear systems for which system theoretic parameters scale dramatically, i.e. exponentially, with the dimension $n$. As a result, the system theoretic quantities $C^{\mathrm{sys}}_{1},\,C^{\mathrm{sys}}_{2}$ might be very large and in fact \emph{dominate} the $\poly(n)$ term in the upper bounds~\eqref{CTRL_eq:existing_upper_bounds}. This phenomenon especially arises in systems which are structurally difficult to control, such as for example underactuated systems. Then, the upper bounds~\eqref{CTRL_eq:existing_upper_bounds} suggest that learning might be difficult for such instances. This brings up the following questions. \emph{Can learning LQR or stabilizing controllers indeed be hard for such systems? How does system structure affect difficulty of learning?}

To answer the first question, we need to establish lower bounds. As we discuss in Section~\ref{CTLR_sec:related_work}, existing lower bounds for online LQR~\citep{simchowitz2020naive} might not always reveal the dependence on control theoretic parameters. \cite{chen2021black}~provided exponential lower bounds for the start-up regret of stabilization. Still, to the best of our knowledge, there are no existing lower bounds for the \emph{sample complexity} of stabilization. Recently, it was shown that the sample complexity of system identification can grow exponentially with the dimension $n$~\citep{tsiamis2021linear}. However, it is not clear if difficulty of identification translates into difficulty of control. Besides, we do not always need to identify the whole system in order to control it~\citep{gevers2005identification}.
To answer the second question, we need to provide upper bounds for several control theoretic parameters.   Our contributions are the following: 

\textbf{Exp($n$) Stabilization Lower Bounds.}
We prove an information-theoretic lower bound for the problem of learning stabilizing controllers, showing that it can indeed be statistically hard for underactuated systems.  In particular, we show that the sample complexity of stabilizing an unknown underactuated linear system can scale exponentially with the state dimension $n$. To the best of our knowledge this is the first paper to address this issue and consider lower bounds in this setting. 

\textbf{Exp($n$) LQR Regret Lower Bounds.}
We show that the regret of online LQR can scale exponentially with the dimension as $\exp(n)\sqrt{T}$. In fact, even common integrator-like systems can exhibit this behavior.  
To prove our result, we leverage recent regret lower bounds~\citep{ziemann2022regret}, which provide a refined analysis linking regret to system theoretic parameters. \cite{chen2021black}~first showed that the start-up cost of the regret (terms of low order) can scale exponentially with $n$. Here, we show that this exponential dependence can also affect multiplicatively the dominant $\sqrt{T}$~term.

\textbf{Exponential Upper Bounds.} Under some additional structural assumptions (bounding systems away from uncontrollability), we provide matching global upper bounds. We show that the sample complexity of stabilization and the regret of online LQR can be at most exponential with the dimension $n$. In fact, we prove a stronger result, that they can be at most exponential with the \emph{controllability index} of the system, which captures the structural difficulty of control -- see Section~\ref{CTRL_sec:controllability}. This implies that if the controllability index is small with respect to the dimension $n$, then learning is guaranteed to be easy.

\subsection{Related Work}\label{CTLR_sec:related_work}
 \textbf{System Identification.}
A related problem is that of system identification, where the learning objective is to recover the model parameters $A,B,H$ from data~\citep{matni2019tutorial}. The sample complexity of system identification was studied extensively in the setting of fully observed linear systems~\citep{dean2017sample,simchowitz2018learning,faradonbeh2018finite,sarkar2018fast,fattahi2019learning,jedra2019sample,wagenmaker2020active,efroni2021sparsity} as well as partially-observed systems~\citep{oymak2018non,sarkar2019finite,simchowitz2019semi,tsiamis2019finite,lee2019non,zheng2020non,lee2020improved,lale2020logarithmic}. Recently, it was shown that the sample complexity of system identification can grow exponentially with the dimension $n$~\citep{tsiamis2021linear}.

\noindent\textbf{Learning Feedback Laws.}
The problem of learning stabilizing feedback laws from data was studied before in the case of stochastic~\citep{dean2017sample,tu2017non,faradonbeh2018stabilization, mania2019certainty} as well as adversarial~\citep{chen2021black} disturbances. The standard paradigm has been to perform system identification, followed by a robust control or certainty equivalent gain design. Prior work is limited to sample complexity upper bounds. To the best of our knowledge, there have been no sample complexity lower bounds.

\noindent\textbf{Online LQR.}
While adaptive control in the LQR framework has a rich history \citep{matni2019self}, the recent line of work on regret minimization in online LQR begins with \cite{abbasi2011regret}. They provide a computationally intractable algorithm based on optimism attaining $O(\sqrt{T})$ regret. Algorithms based on optimism have since been improved and made more tractable  \citep{ouyang2017control, abeille2018improved, abbasi2019model, cohen2019learning, abeille2020efficient}. In a closely related line of work, \cite{dean2018regret} provide an $O(T^{2/3})$ regret bound for robust adaptive LQR control, drawing inspiration from classical methods in system identification and robust adaptive control. It has since been shown that certainty equivalent control, without robustness, can attain the (locally) minimax optimal $O(\sqrt{T})$ regret \citep{mania2019certainty, faradonbeh2020adaptive,lale2020explore, jedra2021minimal}.  In particular, by providing nearly matching upper and lower bounds, \cite{simchowitz2020naive} refine this analysis and establish that the optimal rate, without taking system theoretic quantities into account, is $R_T = \Theta(\sqrt{p^2 n T})$. In this work, we rely on the lower bounds by \cite{ziemann2022regret}, which provide a refined instance specific analysis and also lower bounds for the partially observed setting. Here, we further refine their lower bounds to reveal a sharper dependence of the regret on control theoretic parameters. Hence, we how that certain non-local minimax complexities can be far worse than $R_T = \Omega( \sqrt{p^2 n T})$  and scale exponentially in the problem dimension. Indeed, an exponential start-up cost has already been observed by \cite{chen2021black}, in the case of adversarial disturbances. Here we show that this exponential dependency can persist multiplicatively even for large $T$, in the case of stochastic disturbances. Thus, our results complement the results of~\cite{chen2021black}.

\subsection{Notation}
The transpose of $X$ is denoted by $X'$. For vectors $v\in\R^d$, $\snorm{v}_2$ denotes the $\ell_2$-norm. For matrices $X\in\R^{d_1\times d_2}$, the spectral norm is denoted  by $\snorm{X}_2$. For comparison with respect to the positive semi-definite cone we will use $\succeq$ or $\succ$ for strict inequality. By $\P$ we will denote probability measures and by $\E$ expectation. By $\poly(\cdot)$ we denote a polynomial function of its arguments. By $\exp(\cdot)$ we denote a exponential function of its arguments. 

%% file: formulation.tex
\section{Problem Statement}\label{CTRL_sec:formulation}
System~\eqref{CTRL_eq:system} is characterized by the matrices $A\in\R^{n\times n},\,B\in\R^{n\times p},\,H\in\R^{n\times r}$. We assume that $w_k\sim\mathcal{N}(0,I_r)$ is i.i.d. Gaussian with unit covariance. Without loss of generality the initial state is assumed to be zero $x_0=0$. In a departure from prior work, we do not necessarily assume that the noise is isotropic. Instead, we consider a more general model, where the noise $Hw_k$ is allowed to be degenerate--see also Remark~\ref{CTRL_rem:singular_noise}.

\begin{assumption}\label{CTRL_ass:general_setting}
Matrices $A,B,H$ and the noise dimension $r\le n$ are all unknown. The unknown matrices are bounded, i.e. $\snorm{A}_2,\snorm{B}_2,\snorm{H}_2\le M$, for some positive constant $M\ge 1$. Matrices $B,H$ have full column rank $\rank(B)=p\le n$, $\rank(H)=r\le n$. We also assume that the system is non-explosive $\rho(A)\le 1$.
\end{assumption}
The boundedness assumption on the state parameters allows us to argue about global sample complexity upper bounds. 
To simplify the presentation, we make the assumption that the system is non-explosive $\rho(A)\le 1$.
This setting includes marginally stable systems and is rich enough to provide insights about the difficulty of learning more general systems. 

A policy is a sequence of functions $\pi=\set{\pi_t}_{t=0}^{N-1}$. Every function $\pi_t$ maps previous state-input values $x_0,\dots,x_t,u_0,\dots,u_{t-1}$ and potentially an auxiliary randomization  signal $\AUX$ to the new input $u_t$. Hence all inputs $u_t$ are $\F_t$-measurable, where $\F_t\triangleq \sigma(x_0,\dots,x_t,u_0,\dots,u_{t-1},\AUX)$.
For brevity we will use the symbol $S$ to denote a system $S=(A,B,H)$. Let $\P_{S,\pi}$ ($\E_{S,\pi}(\cdot)$) denote the probability distribution (expectation) of the input-state data when the true system is equal to $S$ and we apply a policy $\pi$.

\subsection{Difficulty of Stabilization}
In the stabilization problem, the goal is to find a state-feedback control law $u=Kx$, where $K$ 
renders the closed-loop system $A+BK$ stable with spectral radius less than one, i.e., $\rho(A+BK)<1$. 
We assume that we collect data $x_0,\dots,x_N,u_0,\dots,u_{N}$, which are generated by system~\eqref{CTRL_eq:system} using any exploration policy $\pi$, e.g. white-noise excitation, active learning etc. Since we care only about sample complexity, the policy is allowed to be maximally exploratory. To make the problem meaningful, we restrict the average control energy.
\begin{assumption}\label{CTRL_ass:input_budget}
The control energy is bounded $\E_{S,\pi} \snorm{u_t}^2_2\le \sigma^2_u$, for some $\sigma_u>0$.
\end{assumption}

Next, we define a notion of learning difficulty for classes of linear systems.  By $\CC_n$ we will denote a class of systems with dimension $n$. We will define as easy, classes of linear system that exhibit $\poly(n)$ sample complexity.

\begin{definition}[Poly$(n)$-stabilizable classes]
	Let $\CC_n$ be a class of systems. Let $\hat{K}_N$ be a function that maps input-state data $(u_{0},x_1),\dots$,$(u_{N-1},x_{N})$ to a control gain. 
	We call the class $\CC_n$ $\poly(n)-$stabilizable if there exists an algorithm $\hat{K}_N$ and an exploration policy $\pi$ satisfying Assumption~\ref{CTRL_ass:input_budget}, such that for any confidence $0\le \delta<1$:
	\begin{align}
		&\sup_{S\in\CC_n}\P_{S,\pi}\paren{\rho(A+B\hat{K}_N)\ge 1}\le \delta\label{CTRL_eq:STAB_objective},\quad
		\text{if}\quad N\sigma^2_u\ge \mathrm{poly}(n,\log 1/\delta, M).
	\end{align}
\end{definition}
Our definition requires both the number of samples and the input energy to be polynomial with the arguments. 
The above class-specific definition can be turned into a local, instance-specific, definition of sample complexity by considering a neighborhood around an unknown system. 
The question then arises whether linear systems are generally poly$(n)$-stabilizable. 
\begin{problem}\label{CTRL_problem:STAB}
Are there linear system classes which are not $\poly(n)$-stabilizable? When can we guarantee $\poly(n)$-stabilizability?
\end{problem}
\subsection{Difficulty of Online LQR}
Consider the LQR objective~\eqref{CTRL_eq:LQR_objective}. Let
the state penalty matrix
$Q\in\R^{n\times n}\succ 0$ be positive definite, with the input penalty matrix $R\in\R^{p\times p}$ also positive definite. When the model is known, the optimal policy is a linear feedback law $\pi_{\star}=\set{K_{\star}x_{k}}^{T-1}_{k=0}$, where $K_{\star}$ is given by
\begin{equation}\label{CTRL_eq:LQR_gain}
K_{\star}=-(B'PB+R)^{-1}B'PA,
\end{equation}
and $P$ is the unique positive definite solution to the Algebraic Riccati Equation (ARE)
\begin{equation}\label{CTRL_eq:LQR_DARE}
P=A'PA+Q-A'PB(B'PB+R)^{-1}B'PA.
\end{equation}
Throughout the paper, we will assume that $Q_T=P$.
If the model of~\eqref{CTRL_eq:system} is unknown,
the goal of the learner is to find an online learning policy $\pi$ that leads to minimum regret $R_T(S)$.
In the setting of online LQR, the data are revealed sequentially, i.e. $x_{t+1}$ is revealed after we select $u_t$. Contrary to the stabilization problem, here we study regret, i.e. there is a tradoff between exploration and exploitation. 
We will define a class-specific notion of learning difficulty based on the ratio between the regret and $\sqrt{T}$.

\begin{definition}[Poly$(n)$-Regret]
Let $\CC_n$ be a class of systems of dimension $n$. We say that the class $\CC_n$ exhibits poly($n$) minimax expected regret if
\begin{equation}
\begin{aligned}
		&\min_{\pi}\sup_{S\in\CC_n}\E_{S,\pi} R_T(S)\le \poly(n,M,\log T)\sqrt{T}+\tilde{O}(1)\label{CTRL_eq:REG_objective_EXPE},
	\end{aligned}
	\end{equation}
	where $\tilde{O}(1)$ hides $\poly\log T$ terms.
\end{definition}
Our definition here is based on expected regret, but we could have a similar definition based on high probability regret guarantees -- see~\cite{dann2017unifying} for distinctions between the two definitions.
Similar to the stabilization problem, we pose the following questions.
\begin{problem}\label{CTRL_problem:REG}
Are there classes of systems for which poly$(n)$-regret is impossible? When is poly$(n)$-regret guaranteed?
\end{problem}

%% file: controllability.tex
\section{Classes with Rich Controllability Structure}\label{CTRL_sec:controllability}
Before we present our learning guarantees, we need to find classes of systems, where learning is meaningful. To make sure that the stabilization and the LQR problems are well-defined, we assume that system~\eqref{CTRL_eq:system} is controllable\footnote{We can slightly relax the condition to $(A,B)$  stabilizable~\citep{lale2020explore,simchowitz2020naive,efroni2021sparsity}. To avoid technicalities we leave that for future work.}. 
\begin{assumption}\label{CTRL_ass:controllability}
System~\eqref{CTRL_eq:system} is $(A,B)$ \emph{controllable}, i.e. matrix
\begin{align}\label{CTRL_eq:controllability_matrix}
	\C_k(A,B)\triangleq \matr{{cccc}B&AB&\cdots&A^{k-1}B}
\end{align}
has full column rank $\rank(\C_k(A,B))=n$, for some $k\le n$.
\end{assumption}
Unsurprisingly, the class of all controllable systems
does not exhibit finite sample complexity/regret, let alone polynomial sample complexity/regret. The main issue is that there exist systems which satisfy the rank condition but are arbitrarily close to uncontrollability.
For example, consider the following controllable system, which we want to stabilize
\begin{equation*}
    x_{k+1}=\matr{{cc}1&\alpha\\0&0}x_k+\matr{{c}0\\1}u_k+w_k.
\end{equation*}
The only way to stabilize the system is indirectly by using the second state $x_{k,2}$, via the coupling coefficient $\alpha$. However, we need to know the sign of $\alpha$. If  $\alpha$ is allowed to be arbitrarily small, i.e. the system is arbitrarily close to uncontrollability, then an arbitrarily large number of samples is required to learn the sign of $\alpha$, leading to infinite complexity. 
To obtain classes with finite sample complexity/regret we need to bound the system instances away from uncontrollability. One way is to consider the least singular value of the controllability Gramian $\Gamma_k(A,B)$ at time $k$:
\begin{equation}\label{CTRL_eq:gramian}
    \Gamma_{k}(A,B)\triangleq \sum_{t=0}^{k-1}A^tBB'(A')^{t}.
\end{equation}
An implicit assumption in prior literature is that $\sigma^{-1}_{\min}(\Gamma_{k}(A,B))\le \poly(n)$. We will not assume this here, since it might exclude many systems of interest, such as integrator-like systems, also known as underactuated systems, or networks~\citep{pasqualetti2014controllability}. Instead, we will relax this requirement to allow richer system structures.

To avoid pathologies, we will lower bound the coupling between states in the case of indirectly controlled systems. To formalize this idea, let us review some notions from system theory. The \emph{controllability index} is defined as follows
\begin{align}\label{eq:controllability_idex}
	\kappa(A,B)\triangleq \min\set{ k\ge 1: \rank(\C_k(A,B))=n },
\end{align}
i.e., it is the minimum time such that the controllability rank condition is satisfied.
It captures the degree of underactuation and reflects the structural difficulty of control. 

Based on the fact that the rank of the controllability matrix at time $\kappa$ is $n$, we can show that the pair $(A,B)$ admits the following canonical representation, under a unitary similarity transformation~\citep{Dooren03}. It is called the Staircase or Hessenberg form of system~\eqref{CTRL_eq:system}.
\begin{proposition}[Staircase form]\label{CTRL_prop:Hessenberg}
	Consider a controllable pair $(A,B)$ with controllability index $\kappa$ and controllability matrix $\C_k$, $k\ge 0$. There exists a unitary similarity transformation $U\in\R^{n\times n}$ such that $U'U=UU'=I$ and:
	\begin{equation}
	\label{CTRL_eq:Hessenberg_form}
	U'B=\matr{{c}B_1\\0\\0\\0\\\vdots\\0},\qquad
U'AU=\matr{{ccccc}A_{1,1}&A_{1,2}&\cdots&A_{1,\kappa-1}&A_{1,\kappa}\\A_{2,1}&A_{2,2}&\cdots&A_{3,\kappa-1}&A_{2,\kappa}\\0&A_{3,2}&\cdots&A_{3,\kappa-1}&A_{3,\kappa}\\0&0&\cdots&A_{4,\kappa-1}&A_{4,\kappa}\\\vdots& & &\vdots&\\0&0&\cdots&A_{\kappa,\kappa-1}&A_{\kappa,\kappa}},
	\end{equation}
	where $A_{i,j}\in \R^{p_i\times p_j}$ are block matrices, with $p_i=\rank(\C_{i})-\rank(\C_{i-1})$, $p_1=p$, $B_1\in\R^{p\times p}$. Matrices $A_{i+1,i}$ have full row rank $\rank(A_{i+1,i})=p_{i+1}$ and the sequence $p_i$ is decreasing.
\end{proposition}
Matrix $U$ is the orthonormal matrix of the QR decomposition of the first $n$ independent columns of $\C_{\kappa}(A,B)$. It is unique up to sign flips of its columns. The above representation captures the coupling between the several sub-states via the matrices $A_{i+1,i}$. It has been used before as a test of controllability~\cite{Dooren03}. 
This motivates the following definition, wherein we bound the coupling matrices $A_{i+1,i}$ away from zero.
\begin{definition}[Robustly coupled systems]
Consider a controllable system $(A,B)$ with controllability index $\kappa$. It is called $\mu-$robustly coupled if and only if for some positive $\mu>0$:
\begin{equation}
  \sigma_{p}(B_{1})\ge\mu,\quad \sigma_{p_{i+1}}(A_{i+1,i})\ge \mu,\,\text{ for all }1\le i\le \kappa-1,
\end{equation}
where $B_1$, $A_{i+1,i}$ are defined as in the Staircase form~\eqref{CTRL_eq:Hessenberg_form}.
\end{definition}
In the previous example, by introducing the $\mu-$robust coupling requirement, we enforce a lower bound on the coupling coefficient $\alpha\ge \mu$, thus, avoiding pathological systems.  

In the following sections, we connect the controllability index to the hardness/ease of control. We prove rigorously why performance might degrade as the index becomes $\kappa=O(n)$, as, e.g., in the case of integrator-like systems or networks.  This cannot be explained based on prior work or based on global lower-bounds on the least singular value of the controllability Gramian. The controllability index and the controllability Gramian are two  different measures that are suitable for different types of guarantees.
The controllability index captures the structural difficulty of control, so it might be more suitable for class-specific guarantees versus instance-specific local guarantees.

%% file: stabilization.tex
\section{Difficulty of Stabilization}\label{sec:stabilization}
In this section, we show that there exist non-trivial classes of linear systems for which the problem of stabilization from data is hard. 
In fact, the class of robustly coupled systems requires at least an exponential, in the state dimension $n$, number of samples.

 \begin{theorem}[Stabilization can be Hard]\label{CTRL_thm:STAB_lower_exponential}
Consider the class $\CC^{\mu}_{n,\kappa}$ of all $\mu$-robustly coupled systems $S=(A,B,H)$ of dimension $n$ and controllability index $\kappa$. Let Assumption~\ref{CTRL_ass:input_budget} hold and let $\mu<1$. Then, for any stabilization algorithm, the sample complexity is exponential in the index $\kappa$. For any confidence $0\le \delta<1/2$ the requirement
\begin{align}
		&\sup_{S\in\CC^{\mu}_{n,\kappa}}\P_{S,\pi}\paren{\rho(A+B\hat{K}_N)\ge 1}\le \delta \nonumber
	\end{align}
is satisfied only if
\[
N\sigma^2_u  \ge \frac{1}{2}\paren{\frac{1}{\mu}}^{2\kappa-2}\paren{\frac{1-\mu}{\mu}}^{2}\log\frac{1}{3 \delta}.
\]
\end{theorem}
Theorem~\ref{CTRL_thm:STAB_lower_exponential} implies that system classes with large controllability index, e.g. $\kappa=n$, suffer in general from sample complexity which is exponential with the dimension $n$.
In other words, learning difficulty arises in the case of under-actuated systems. 
Only a limited number of system states are directly driven by inputs and the remaining states are only indirectly excited, 
leading to a hard learning and stabilization problem. 
Consider now systems
\begin{equation}\label{CTRL_eq:difficult_example_STAB}
\begin{aligned}
  S_i:\qquad  	x_{k+1}=\matr{{ccccc}1 &\alpha_i\mu&0&\cdots&0\\0& 0&\mu&\cdots&0\\& &\ddots &\ddots&\\0&0&0&\cdots&\mu\\0&0&0&\cdots&0}x_k+\matr{{c}0\\0\\\vdots\\0\\\mu}u_k+ \matr{{c}1\\0\\\vdots\\0\\0}w_k,\,i\in\set{1,2},
\end{aligned}
\end{equation}
where $0<\mu<1$, $\alpha_1=1$, $\alpha_2=-1$.
Systems $S_1$, $S_2$ are almost identical with the exception of element $A_{12}$ where they have different signs. Both systems have one marginally stable mode corresponding to state $x_{k,1}$. The only way to stabilize $x_{k,1}$ with state feedback is indirectly, via $x_{k,2}$. Given system $S_1$, since $\alpha_1\mu>0$, it is necessary that the first component of the gain is negative $\hat{K}_{N,1}<0$. This follows from the Jury stability criterion, a standard stability test in control theory~\citep[Ch. 4.5]{fadali2013digital}. 
Let $\phi_1(z)=\det(zI-A_1-B\hat{K}_N)$ be the characteristic polynomial of system $S_1$. Then one of the necessary conditions in Jury's criterion requires:
\[
\phi_1(1)>0,
\]
which can only be satisfied if $\hat{K}_{N,1}<0$  (see Appendix~\ref{CTRL_app_sec:STAB_lower_bounds} for details).
On the other hand, we can only stabilize $S_2$
if $\hat{K}_{N,1}>0$. Hence, the only way to stabilize the system is to identify the sign of $\alpha_i$. In other words, we transform the stabilization problem into a system identification problem. However, identification of the correct sign is very hard since the excitation of $x_{k,2}=\mu^{n-1}u_{k-n+1}$ scales with $\mu^{n-1}$. 
The proof relies on Birgé's inequality~\citep{boucheron2013concentration}.
In Section~\ref{CTRL_app_sec:STAB_lower_bounds} we construct a slightly more general example with non-zero diagonal elements.
Our construction relies on the fact that $\mu<1$. It is an open question whether we can construct hard learning instances for $\mu\ge 1$.

One insight that we obtain from the above example is that lack of excitation might lead to large sample complexity of stabilization. In particular, this can happen when we have an unstable/marginally stable mode, which can only be controlled via the system identification bottleneck, like $A_{1,2}$ in the above example. 

\begin{remark}[Singular noise]\label{CTRL_rem:singular_noise}
Our stabilization lower bound exploits the fact that the constructed system~\eqref{CTRL_eq:difficult_example_STAB} has low-rank noise, such that system identification is hard. It is an open problem whether we can construct examples of systems that are not $\poly(n)-$stabilizable even though they are excited by full-rank noise. Nonetheless, in our regret lower bounds, we allow the noise to be full-rank.
\end{remark}
\subsection{Sample complexity upper bounds}\label{CTRL_sec:STAB_upper_bounds}
 As we show below, sample complexity cannot be worse than exponential under the assumption of robust coupling. If the exploration policy is a white noise input sequence, then using a least squares identification algorithm~\citep{simchowitz2018learning}, and a robust control design scheme~\citep{dean2017sample}, the sample complexity can be upper bounded by a function which is at most exponential with the dimension $n$. In fact, we provide a more refined result, directly linking sample complexity to the controllability index $\kappa$. Our proof relies on bounding control theoretic quantities like the least singular value of the controllablility Gramian.
The details of the proof and the algorithm can be found in Section~\ref{CTRL_app_sec:STAB_upper_bounds}. 

\begin{theorem}[Exponential Upper Bounds]\label{CTRL_thm:upper_bounds_STAB}
Consider the class $\CC^{\mu}_{n,\kappa}$ of all $\mu$-robustly coupled systems $S=(A,B,H)$ of dimension $n$ and controllability index $\kappa$. Let Assumption~\ref{CTRL_ass:input_budget} hold. Then, the sample complexity is at most exponential with $\kappa$. There exists an exploration policy $\pi$ and algorithm $\hat{K}_N$ such that for any $\delta<1$:
\begin{align*}
		&\sup_{S\in\CC^{\mu}_{n,\kappa}}\P_{S,\pi}\paren{\rho(A+B\hat{K}_N)\ge 1}\le \delta,\quad \text{if} \quad
N\sigma^2_u \ge   \poly\paren{\Big(\frac{M}{\mu}\Big)^\kappa,M^{\kappa},n,\log 1/\delta}.
\end{align*}
\end{theorem}
Assume that the constants $\mu$ and $M$ are dimensionless. Then, our upper and lower bounds match qualitatively with respect to the dependence on $\kappa$. Theorem~\ref{CTRL_thm:upper_bounds_STAB} implies that if the degree of underactuation is mild, i.e. $\kappa=O(\log n)$, then robustly coupled systems are guaranteed to be poly$(n)$-stabilizable. 
Our upper bound picks up a dependence on the quantity  $M/\mu$. Recall that $M$ upper-bounds the norm of $A$. Hence, it captures a notion of sensitivity of the dynamics $A$ to inputs/noise. In the lower bounds only the coupling term $\mu$ appears. It is an open question to prove or disprove whether the sensitivity of $A$ affects stabilization or it is an artifact of our analysis. 
Another important open problem is to determine the optimal constant that multiplies $\kappa$ in the exponent. Our lower bound suggests that the exponent can be at least of the order of $2$ times $\kappa$. In our upper bounds, by following the proof, we get an exponent which is larger than $2$.

%% file: onlineLQR.tex
\section{Difficulty of online LQR}\label{sec:onlineLQR}
In the following theorem, we prove that classes of robustly coupled systems can exhibit minimax expected regret which grows at least exponentially with the dimension $n$.
Let $\CC^{\mu}_{n,\kappa}$ denote the class of $\mu$-robustly coupled systems $S=(A,B,H)$ of state dimension $n$ and controllability index $\kappa$. Define the $\epsilon$-dilation $\CC^{\mu}_{n,\kappa}(\epsilon)$ of $\CC^{\mu}_{n,\kappa}$ as
\[
\CC^{\mu}_{n,\kappa}(\epsilon)\triangleq \set{(A,B,H):\: \snorm{\matr{{cc}A-\tilde{A}&B-\tilde{B}}}_2\le \epsilon,\text{ for some }(\tilde{A},\tilde{B},H)\in \CC^{\mu}_{n,\kappa}},
\]
which consists of every system in $\CC^{\mu}_{n,\kappa}$ along with its $\epsilon-$ball around it.
 \begin{theorem}[Exponential Regret Lower Bounds]\label{CTRL_thm:REG_lower_exponential}
Consider the class $\CC^{\mu}_{n,\kappa}$ of all $\mu$-robustly coupled systems $S=(A,B,H)$ of state dimension $n$ and controllability index $\kappa$, with $\kappa\le n-1$. For every $\epsilon>0$ define the $\epsilon$-dilation $\CC^{\mu}_{n,\kappa}(\epsilon)$. Let $Q_T=P$, the solution to the ARE \eqref{CTRL_eq:LQR_DARE}, and assume $\mu<1$. Let $0<\alpha<1/4$.  For any policy $\pi$
\begin{equation*}
\begin{aligned}
		&\liminf_{T\rightarrow \infty}\sup_{S\in\CC^{\mu}_{n,\kappa}(T^{-\alpha})}\E_{S,\pi} \frac{R_T(S)}{\sqrt{T}}\ge \frac{1}{4\sqrt{n}} 2^{\frac{\kappa-1}{2}}.
	\end{aligned}
	\end{equation*}
\end{theorem}
When the controllability index is large, e.g. $\kappa=n$, then the lower bounds become exponential with $n$. Hence, achieving poly($n$)-regret is impossible in the case of general linear systems.
In general, learning difficulty depends on fundamental control theoretic parameters, i.e. on the solution $P$ to the  ARE~\eqref{CTRL_eq:LQR_DARE} or the steady-state covariance of the closed-loop system, both of which can scale exponentially with the controllability index. 
Existing regret upper-bounds depend on such quantities in a transparent way~\cite{simchowitz2020naive}. Here, we reveal the dependence on such parameters in the regret lower-bounds as well (Lemma~\ref{CTRL_lem:modular_bound_two_subsystems}).

Let us now explain when learning can be difficult.
Consider the following $1-$strongly coupled system, which consists of two independent subsystems
\begin{equation}\label{CTRL_eq:REG_difficult_example_integrator}
A=\matr{{c|ccccc}0&0&0&&0&0\\\hline 0&1&1&&0&0\\& &&\ddots&\\0&0&0& &1&1\\0&0&0& &0&1},\,B=\matr{{c|c}1&0\\0&0\\\vdots\\0&1}u_k,\,H=I_n,\,Q=I_n,\,R=I_2,
\end{equation}
where the first subsystem is a memoryless system, while the second one is the discrete integrator of order $n-1$.
Since the sub-systems are decoupled, the optimal LQR controller will also be decoupled and structured
\[
K_{\star}=\matr{{cc}0 & 0\\0&K_{\star,0}},
\]
where $K_{\star,0}$ is the optimal gain of the second subsystem.
The first subsystem (upper-left) is memoryless and does not require any regulation, that is, $[K_{\star}]_{11}=0$. 

Consider now a perturbed system $\tilde{A}=A-\Delta K_{\star}$, $\tilde{B}=B+\Delta$, for some $\Delta\in\R^{p\times n}$. Such perturbations are responsible for the $\sqrt{T}$ term in the regret of LQR~\citep{simchowitz2020naive,ziemann2022regret}; systems $(A,B)$ and $(\tilde{A},\tilde{B})$ are indistinguishable under the control law $u_t=K_{\star}x_{t}$ since
$
A+BK_{\star}=\tilde A+\tilde BK_{\star}.
$
Now, informally, to get an $\exp(n)\sqrt{T}$ regret bound it is sufficient to satisfy two conditions: i) the system is sensitive to inputs or noise, in the sense that any exploratory signal can incur extra cost, which grows exponentially with $n$. ii) the difference $\tilde{A}-A$, $\tilde{B}-B$ is small enough, i.e. polynomial in $n$, so that identification of $\Delta$ requires significant deviation from the optimal policy.

The $n-1$-th integrator is very sensitive to inputs or noises. As inputs $u_{k,2}$ and noises $w_k$ get integrated $(n-1)$-times, this will result in accumulated values that grow exponentially as we move up the integrator chain. Hence, the first informal condition is satisfied. To satisfy the second condition we let
the perturbation $\Delta$ have the following structure
\begin{equation}\label{CTRL_eq:perturbation_structure}
\Delta=\matr{{cc}0&0\\\Delta_1&0},
\end{equation}
where we only perturb the matrix of the first input $u_{k,1}$. 
By using two subsystems and the above construction, we make it harder to detect $\Delta$. In particular, because of the structure of the system ($[K_{\star}]_{11}=0$) and the perturbation $\Delta$, we have
$\tilde{A}=A-\Delta K_{\star}=A$. Hence $\snorm{\matr{{cc}A&B}-\matr{{cc}\tilde{A}&\tilde{B}}}_2= \snorm{\Delta}_2\le  \poly(n)\snorm{\Delta}_2 ,$ i.e., the perturbed system does not lie too far away from the nominal one. This last condition might be crucial. If $\snorm{\Delta K_{\star}}\ge \exp(n)\snorm{\Delta}_2$, then it might be possible to distinguish between $(A,B)$ and $(\tilde{A},\tilde{B})$ without deviating too much from the optimal policy. This may happen if we use only one subsystem, since $\snorm{K_{\star,0}}_2$ might be large. By using two subsystems, we cancel the effect of $K_{\star,0}$ in $\Delta K_{\star}$. 

In the stabilization problem, we show that the lack of excitation during the system identification stage might hurt sample complexity. Here, we show that if a system is too sensitive to inputs and noises, i.e. some state subspaces are too easy to excite, this can lead to large regret. Both lack of excitation and too much excitation of certain subspaces can hurt learning performance. This was observed before in control~\citep{skogestad1988robust}. 

\subsection{Sketch of Lower Bound Proof}
Let $S_0=(A_0,B_0,I_{n-1})\in \CC^{\mu}_{n-1,\kappa}$ be a $\mu-$robustly coupled system of state dimension $n-1$, input dimension $p-1$ and controllability index $\kappa\le n-1$. Let $P_0$ be the solution of the Riccati equation for $Q_0=I_{n-1}$, $R_0=I_{p-1}$, with $K_{\star,0}$ the corresponding optimal gain. 
Define the steady-state covariance of the closed-loop system
\begin{equation}\label{CTRL_eq:Covariance_steady_state}
\Sigma_{0,x}=(A_0+B_0K_{\star,0})\Sigma_{0,x}(A_0+B_0K_{\star,0})'+I_{n-1}.
\end{equation}
Now, consider the composite system:
\begin{equation}\label{CTRL_eq:composite_system}
A=\matr{{cc}0&0\\0&A_0},\, B=\matr{{cc}1&0\\0&B_0},\, H=I_n,
\end{equation}
with $Q=I_n,\,R=I_p$. 
Let $\Delta$ be structured as in~\eqref{CTRL_eq:perturbation_structure}, for some arbitrary  $\Delta_1$ of unit norm $\snorm{\Delta_1}_2=1$.  The Riccati matrix of the composite system is denoted by $P$ and the corresponding gain by $K_{\star}$. 
Consider the parameterization:
\begin{equation}\label{CTRL_eq:parameterized_system_family}
A(\theta)=A-\theta \Delta K_{\star}, \qquad B(\theta)=B+\theta \Delta,
\end{equation}
for any $\theta \in\R$. 
Let $\mathcal{B}(\theta,\epsilon)$ denote the open Euclidean ball of radius $\epsilon$ around $\theta$. 
For every $\epsilon>0$, define the local class of systems around $S$ as $\CC_S({\epsilon})\triangleq\set{(A(\theta),B(\theta),I_n),\,\theta\in\mathcal{B}(0,\epsilon)}$. 
Based on the above construction and Theorem~1 of~\cite{ziemann2022regret}, a general information-theoretic regret lower bound, we prove the following lemma.

\begin{lemma}[Two-Subsystems Lower Bound]\label{CTRL_lem:modular_bound_two_subsystems}
Consider the parameterized family of linear systems defined in~\eqref{CTRL_eq:parameterized_system_family}, for $n,p \ge 2$ where $\Delta$ is structured as in~\eqref{CTRL_eq:perturbation_structure}. Let $Q=I_n$, $R=I_p$. Let $Q_T=P(\theta)$, where $P(\theta)$ is the solution to the Riccati equation for $(A(\theta),B(\theta))$. Then, for any policy $\pi$ and any $0<a<1/4$ the expected regret is lower bounded by
\begin{align*}
&\lim\inf_{T\rightarrow \infty}\sup_{\hat{S}\in \CC_S({T^{-a}})}\E_{\hat{S},\pi}\frac{R_{T}(\hat{S})}{\sqrt{T}}\ge \frac{1}{4\sqrt{n}}\sqrt{\Delta'_1 P_0 \clint{\Sigma_{0,x}-I_{n-1}}P_0\Delta_1}.
\end{align*}
\end{lemma}
Optimizing over $\Delta_1$, we obtain a lower bound on the order of $\snorm{P_0 \clint{\Sigma_{0,x}-I_{n-1}}P_0}_2$. 
What remains to show is that for the $(n-1)$-th order integrator (second subsystem in~\eqref{CTRL_eq:REG_difficult_example_integrator}) the product $\snorm{P_0 \clint{\Sigma_{0,x}-I_{n-1}}P_0}_2$ is exponentially large with $n$.
\begin{lemma}[System Theoretic Parameters can be Large]\label{CTRL_lem:integrator_system_theoretic_parameters}
Consider the $(n-1)-th$ order integrator (second subsystem in~\eqref{CTRL_eq:REG_difficult_example_integrator}).
Let $P_0$ be the Riccati matrix for $Q_0=I_{n-1},R_0=1$, with $K_{\star,0}$, $\Sigma_{0,x}$ the corresponding LQR control gain and steady-state covariance. Then 
\[\snorm{P_0 \clint{\Sigma_{0,x}-I_{n-1}}P_0}_2\ge \sum_{j=1}^{n-1} \sum_{i=0}^{j}\binom{j}{i}^2\ge 2^{n-1}\]
\end{lemma}
Our lemma shows that control theoretic parameters can scale exponentially with the dimension $n$. 
The $(n-1)-$th order integrator is a system which is mildly unstable. In Section~\ref{CTRL_sec:REG_stable}, we show that \textbf{stable} systems can also suffer from the same issue. 

\subsection{Regret Upper Bounds}
Similar to the stabilization problem, we show that under the assumption of robust coupling, the regret cannot be worse than $\exp(\kappa)\sqrt{T}$ with high probability. 
As we prove in~Lemma~\ref{CTRL_lem:Riccati_Upper}, the solution $P$ to the Riccati equation has norm $\snorm{P}_{2}$ that scales at most exponentially with the index $\kappa$ in the case of robustly-coupled systems. This result combined with the regret upper bounds of~\cite{simchowitz2020naive},  give us the following result.

\begin{theorem}[Exponential Upper Bounds]
Consider a $\mu$-robustly coupled system $S=(A,B,H)$ of dimension $n$, controllability index $\kappa$. Assume that we are given an initial stabilizing gain $K_0$. Let $Q=I_n$, $R\succeq I_p$, and $Q_T=0$. Assume that the noise is non-singular $HH'=I_n$\footnote{It is possible to relax some of the assumptions on the noise--see~
\cite{simchowitz2020naive}}. Let $\delta\in (0,1/T)$. Using the Algorithm~1 of~\cite{simchowitz2020naive} with probability at least $1-\delta$:
\[
R_T(A,B)\le \poly(n,\big(\frac{M}{\mu}\big)^{\kappa},M^{\kappa},\log 1/\delta)\sqrt{T}+\poly(n,\big(\frac{M}{\mu}\big)^{\kappa},M^{\kappa},\log 1/\delta,P(K_0)),
\]
where $P(K_0)=(A+BK_0)'P(K_0)(A+BK)+Q+K'_0RK_0$.
\end{theorem}
The result follows immediately by our Lemma~\ref{CTRL_lem:Riccati_Upper} and the upper bounds of Theorem~2 in~\cite{simchowitz2020naive}.
Assuming that the plant sensitivity $M$ and the coupling coefficient $\mu$  are dimensionless, then if we have a mild degree of underactuation, i.e. $\kappa=O(\log n)$, we get poly($n$)-regret with high probability.  Note that the above guarantees are for high probability regret which is not always equivalent to expected regret~\citep{dann2017unifying}. 
Our upper-bounds are almost global for all robustly coupled systems, in the sense that the dominant $\sqrt{T}$-term is globally bounded. To provide truly global regret guarantees it is sufficient to add an initial exploration phase to Algorithm~1 of \cite{simchowitz2020naive}, which first learns a stabilizing gain $K_0$. For this stage we could use the results of Section~\ref{CTRL_sec:STAB_upper_bounds}, and Section~\ref{CTRL_app_sec:STAB_upper_bounds}. We leave this for future work.

%% file: conclusion.tex
\section{Conclusion}\label{CRTL_sec:conclusion}
We prove that learning to control linear systems can be hard for non-trivial system classes. The problem of stabilization might require sample complexity which scales exponentially with the system dimension $n$. Similarly, online LQR might exhibit regret which scales exponentially with $n$. This difficulty arises in the case of underactuated systems. Such systems are structurally difficult to control; they can be very sensitive to inputs/noise or very hard to excite. If the system is robustly coupled and has a mild degree of underactuation (small controllability index), then we can guarantee that learning will be easy. 

We stress that system theoretic quantities might not be dimensionless. On the contrary, they might grow very large with the dimension and dominate any poly$(n)$ terms. Hence, going forward, an important direction of future work is to find policies with optimal dependence on such system theoretic quantities. Although the optimal dependence is known for the problem of system identification~\citep{simchowitz2018learning,jedra2019sample}, it is still not clear what is the optimal dependence in the case of control. For example, an interesting open problem is to find the optimal dependence of the regret $R_T$ on the Riccati equation solution $P$. For the problem of stabilization, it is open to find how sample complexity optimally scales with the least singular value of the controllability Gramian.





%% file: appendix.tex
\newpage
\tableofcontents
\newpage
\appendix
\addcontentsline{toc}{part}{Appendix}
\section{System Theoretic Preliminaries}
In this section, we review briefly some system theoretic concepts. A system $(A,B)\in\R^{n\times (n+p)}$ is \textbf{controllable} if and only if the controllability matrix
\[
	\C_k(A,B)=\matr{{cccc}B&AB&\cdots&A^{k-1}B}
\]
has full column rank for some $k\le n$. The minimum such index $\kappa$ that the rank condition is satisfied is called the controllability index, and it is always less or equal than the state dimension $n$.
A system $(A,B)$ is called \textbf{stabilizable} if and only if there exists a matrix $K\in\R^{p\times n}$ such that $A+BK$ is stable, i.e. has spectral radius $\rho(A+BK)$. Any controllable system is also stabilizable.
A system $(A',B')$ is called \textbf{observable} if and only if $(A,B)$ is controllable. Similarly $(A',B')$ is \textbf{detectable} if and only if $(A,B)$ is stabilizable.

Let $A$ be stable ($\rho(A)<1$) and consider the transfer matrix $(zI-A)^{-1},z\in\mathcal{C}$ in the frequency domain. The $\mathcal{H}_{\infty}$-norm is given by
\[
\snorm{(zI-A)^{-1}}_{\mathcal{H}_{\infty}}=\sup_{\abs{z}=1}\snorm{(zI-A)^{-1}}_2.
\]
Using the identity $(I-D)^{-1}=I+D+D^2\dots$ for $\rho(D)<1$, we can upper bound the $\mathcal{H}_{\infty}$-norm by \[\snorm{(zI-A)^{-1}}_{\mathcal{H}_{\infty}}\le \sum_{t=0}^{\infty} \snorm{A^t}_2.\]

\subsection{Properties of the Riccati Equation}\label{CTRL_app_sec:Riccati}
Consider the infinite horizon LQR problem defined in~\eqref{CTRL_eq:LQR_objective}. Let $(A,B)$ be controllable and assume that $Q\succ 0$ is positive semi-definite and $R\succ 0$ is positive definite.
As we stated in Section~\ref{CTRL_sec:formulation}, 
the optimal policy $K_{\star}x_k$ has the following closed-form solution
\begin{equation*}
K_{\star}=-(B'PB+R)^{-1}B'PA,
\end{equation*}
where $P$ is the unique positive definite solution to the \textbf{Discrete Algebraic Riccati Equation}
\begin{equation*}
P=A'PA+Q-A'PB(B'PB+R)^{-1}B'PA.
\end{equation*}
Moreover, $A+BK_{\star}$ is stable, i.e. $\rho(A+BK_{\star})<1$. 
The above solution is well-defined under the conditions of $(A,B)$ controllable, $Q\succ 0$, $R\succ 0$.
Note that we can relax the conditions to $Q\succeq 0$ being positive semi-definite, $(A,Q^{1/2})$ detectable, and $(A,B)$ stabilizable, which is a well-known result in control theory~\citep[Th. 3.1]{chan1984convergence}.

Consider now the \textbf{finite-horizon} LQR problem, under the same assumptions of $(A,B)$ controllable, $Q\succ 0$, and $R\succ 0$
\begin{equation}
  J_T^*(S)\triangleq \min_{\pi} \E_{S,\pi} \clint{\sum^{T-1}_{t=0}(x'_tQx_t+u'_tRu_t)+x'_TQ_Tx_T}.
\end{equation}
The optimal policy is a feedback law $K_{t}x_t$, $t\le T-1$, with time varying gains. The gains satisfy the following closed-form expression
\begin{equation*}
K_{t}=-(B'P_{t+1}B+R)^{-1}B'P_{t+1}A,
\end{equation*}
where $P_t$ satisfies the \textbf{Riccati Difference Equation}
\begin{equation*}
P_{t}=A'P_{t+1}A+Q-A'P_{t+1}B(B'P_{t+1}B+R)^{-1}B'P_{t+1}A,\, P_T=Q_T.
\end{equation*}
It turns out that as we take the horizon to infinity $T\rightarrow \infty$, then we get $\lim_{T\rightarrow \infty} P_k=P$ exponentially fast, for any fixed $k$, where $P$ is the positive definite solution to the Algebraic Riccati Equation. The convergence is true under the conditions of $(A,B)$ controllable, $Q\succ 0$, $R\succ 0$. Again we could relax the conditions to $Q\succeq 0$ being positive semi-definite, $(A,Q^{1/2})$ detectable, and $(A,B)$ stabilizable~\citep[Th. 4.1]{chan1984convergence}.
Note that if we select the terminal cost $Q_T=P$, then trivially $P_t=P$ for all $t\le T$, and we recover the same controller as in the infinite horizon case.

Finally, a nice property of the Riccati recursion is that the right-hand side is order-preserving with respect to the matrices $P,Q$. In particular, define the operator:
\[
g(X,Y)=A'XA+Y-A'YB(B'XB+R)^{-1}B'YA.
\]
Then, if $X_1\succeq X_2$, we have that $g(X_1,Y)\succeq g(X_2,Y)$~\citep[Ch. 4.4]{anderson2005optimal}. Similarly, if $Y_1\succeq Y_2$ then $g(X,Y_1)\succeq g(X,Y_2)$.

\section{System Theoretic Bounds for Robustly Coupled Systems}
The first result lower bounds the least singular value of the controllability Gramian in terms of the sensitivity $M$, the coupling coefficient $\mu$, and the controllability index $\kappa$ of the system. 
\begin{theorem}[Gramian lower bound~\citep{tsiamis2021linear}]\label{CTRL_thm:gramian_lower_bound}
	Consider a system $(A,B,H)$ that satisfies Assumption~\ref{CTRL_ass:general_setting}, with $\kappa$ its controllability index. Assume that $(A,B)$ is $\mu$-robustly coupled. Then, the least singular value of the Gramian $\Gamma_{\kappa}=\Gamma_{\kappa}(A,B)$ is lower bounded by:
	\[
	\sigma_{\min}^{-1}(\Gamma_{\kappa})\le \mu^{-2}\bigg(\frac{3M}{\mu}\bigg)^{2\kappa}.
	\]
\end{theorem}
\begin{proof}
The result follows from Theorem 5 in~\cite{tsiamis2021linear}. The theorem statement requires a different condition, called robust controllability. However, the proof still goes through if we have $\mu-$robust coupling instead. 
Recall that $\C_{\kappa}=\C_{\kappa}(A,B)$ is the controllability matrix~\eqref{CTRL_eq:controllability_matrix} of $(A,B)$ at $\kappa$.
Following the proof in~\citep{tsiamis2021linear}, we arrive at
\[
\sqrt{\sigma_{\min}(\Gamma_{\kappa}})\le\snorm{\C^{\dagger}_\kappa}_2 \le \snorm{\Xi^{\kappa-1}}_2\snorm{\alpha}_{2},
\]
where
\[
\Xi=\matr{{ccc}1&1&\mu^{-1}\\\frac{M}{\mu}&\frac{2+M}{\mu}&\frac{M}{\mu}\\0&0&\mu^{-1}},\,\alpha=\matr{{c}\frac{1}{\mu}\\\frac{M}{\mu^2}\\ \frac{1}{\mu}}.
\]
The result follows from the crude bounds $\snorm{\Xi}_2\le 3 M/\mu $, $\snorm{\alpha}_2\le \sqrt{3}M/\mu^{-2}$ where we assumed that $M>1$.
\end{proof}

The following result, upper bounds the solution $P$ to the LQR Riccati equation in terms of the sensitivity $M$, the coupling coefficient $\mu$, and the controllability index $\kappa$ of the system. 
\begin{lemma}[Riccati Upper Bounds]\label{CTRL_lem:Riccati_Upper}
Let the system $(A,B)\in\R^{n\times (n+p)}$ be controllable and $\mu-$robustly coupled with controllability index $\kappa$. Let $R\in\R^{p\times p}$ be positive definite and $Q\in\R^{n\times n}$ be positive semi-definite. Assume $T>\kappa$ and consider the Riccati difference equation:
\[
P_{k-1}=A'P_{k}A+Q-A'P_kB(B'P_kB+R)^{-1}B'P_kA,\: P_T=Q.
\]
Then, the Riccati matrix evaluated at time $0$ is upper-bounded by
\[
\snorm{P_0}_2\le \poly\Big(\big(\frac{M}{\mu}\big)^{ \kappa},M^\kappa,\kappa,\snorm{Q}_2,\snorm{R}_2\Big).
\]
As a result, if $Q\succ 0$, then the unique positive definite solution $P$ of the algebraic Riccati equation:
\[
P=A'PA+Q-A'PB(B'PB+R)^{-1}B'PA
\]
satisfies the same bound
\[
\snorm{P}_2\le \poly\Big(\big(\frac{M}{\mu}\big)^{ \kappa},M^\kappa,\kappa,\snorm{Q}_2,\snorm{R}_2\Big).
\]
\end{lemma}
\begin{proof}
The optimal policy of the LQR problem does not depend on the noise. Even for deterministic systems, the optimal policy still have the same form $u_t=K_{\star}x_t$. This property is known as certainty equivalence~\citep[Ch. 4]{bertsekas2017dynamic}. In fact, for deterministic systems, the cost of regulation is given explicitly by $x_0'Px_0$. We leverage this idea to upper bound the stabilizing solution of the Riccati equation $P$. 

\noindent\textbf{Step a) Noiseless system upper bound.}
Consider the noiseless version of system~\eqref{CTRL_eq:system}
\begin{equation}\label{CTRL_eq:noiseless_system}
    x_{k+1}=Ax_k+Bu_k,\quad \snorm{x_0}_2=1.
\end{equation}
Let $u_{0:t}$ be the shorthand notation for
\[
u_{0:t}=\matr{{c}u_{t}\\\vdots\\u_0}.
\]
Consider the deterministic LQR objective
\begin{align*}
    \min_{u_{0:T-1}}&\quad J(u_{0:T-1})\triangleq x_T'Qx_T+\sum_{k=0}^{N-1}x_k'Qx_k+u_k'Ru_k\\
    \mathrm{s.t.}& \quad \text{dynamics~\eqref{CTRL_eq:noiseless_system}}.
\end{align*}
The optimal cost of the problem is given by~\citep[Ch. 4]{bertsekas2017dynamic}
\[
\min_{u_{0:T-1}} J(u_{0:T-1})=x_0'P_0x_0,
\]
where $P_0$ is the value of $P_t$ at time $t=0$. Let $u_{0:T-1}$ be any input sequence. Immediately, by optimality, we obtain an upper bound for the Riccati matrix $P_0$:
\begin{equation}\label{CTRL_eq:deterministic_P_bound}
x_0'P_0x_0 \le J(u_{0:T-1}).
\end{equation}
Hence, it is sufficient to find a suboptimal policy that incurs a cost which is at most exponential with the controllability index $\kappa$.

\noindent\textbf{Step b) Suboptimal Policy.} It is sufficient to drive the state $x_{\kappa}$ to zero at time $\kappa$ with minimum energy $u_{0:\kappa-1}$ and then keep $x_{t+1}=0$, $u_{t}=0$, for $t\ge \kappa$. Recall that $\C_k$ is the controllability matrix at time $k$. By unrolling the state $x_{\kappa}$:
\[
x_{\kappa}=A^{\kappa}x_0+\C_{\kappa}u_{0:\kappa-1}.
\]
To achieve $x_{\kappa}=0$, it is sufficient to apply the minimum norm control
\[
u_{0:\kappa-1}=-\C^{\dagger}_{\kappa}A^{\kappa}x_0,
\]
which leads to input penalties
\[
\sum_{k=0}^{T-1}u'_{k}R u_k\le \snorm{R}_2\sigma^{-1}_{\min}(\Gamma_{\kappa})M^{2\kappa},
\]
where we used the fact that $\snorm{x_0}_2=1$.
For the state penalties, we can write in batch form
\[
x_{1:\kappa}\triangleq\matr{{c}x_{\kappa}\\\vdots\\x_1}=\matr{{cccc}B&AB&\cdots&A^{\kappa-1}B\\0&B&\cdots&A^{\kappa-2}B\\\vdots\\0&0&\cdots&B}u_{0:\kappa-1}+\matr{{c}A^{\kappa}\\A^{\kappa-1}\\\vdots\\A}x_0.
\]
Exploiting the Toeplitz structure of the first matrix above and by Cauchy-Schwartz
\begin{align*}
    \sum_{t=0}^{T}x'_{t}Qx_t&\le \snorm{Q}_2 (\snorm{x_{1:\kappa}}^2_2+1)\\
    &\le 2\snorm{Q}_2\big((\sum^{\kappa-1}_{t=0} \snorm{A^tB}_2)^2\snorm{u_{0:\kappa-1}}^2_2+\sum_{t=0}^{\kappa}\snorm{A^t}_2\big)\\
    &\le 2\kappa^2\snorm{Q}_2(M^{4\kappa} \snorm{R}_2\sigma^{-1}_{\min}(\Gamma_{\kappa})+M^{2\kappa}).
\end{align*}
Putting everything together and since $x_0$ is arbitrary, we finally obtain
\begin{equation}\label{CTRL_eq:P_upper_bound}
\snorm{P_0}_2\le \frac{\snorm{R}_2}{\sigma_{\min}(\Gamma_{\kappa})}(M^{2\kappa}+2\kappa^2\snorm{Q}_2 M^{4\kappa})+2\kappa^2\snorm{Q}_2 M^{2\kappa}.
\end{equation}
The result for $P_0$ now follows from Theorem~\ref{CTRL_thm:gramian_lower_bound}.

\noindent\textbf{Step c) Steady State Riccati.}
If the pair $(A,Q^{1/2})$ is observable, then from standard LQR theory-see Section~\ref{CTRL_app_sec:Riccati}, $\lim_{T\rightarrow \infty} P_0=P$ and the bound for $P$ follows directly.
\end{proof}
Similar results have been reported before~\citep{cohen2018online,chen2021black}. However, instead of $\kappa$ and $(M/\mu)^{\kappa}$, the least singular value $\sigma^{-1}_{\min}(\Gamma_{k})$ shows up in the bounds, for some $k\ge \kappa$.

Finally, based on Lemmas B.10, B.11 of~\cite{simchowitz2020naive}, we provide some upper bounds on the $\mathcal{H}_{\infty}-$norm of the closed loop response $(zI-A+BK)^{-1}$, where $K$ is the control gain of the optimal LQR controller for some $Q$ and $R$.
\begin{lemma}[LQR Robustness Margins]\label{CTRL_lem:margins}
Let the system $(A,B)\in\R^{n\times (n+p)}$ be controllable and $\mu-$robustly coupled. Let $R=I_p,\,Q=I_n$. Let $P$ be the stabilizing solution of the algebraic Riccati equation:
\[
P=A'PA+Q-A'PB(B'PB+R)^{-1}B'PA
\]
with $K_{\star}$ the respective control gain
$
K_{\star}=-(B'PB+R)^{-1}B'PA.
$
The spectral radius and the $\mathcal{H}_{\infty}$-norm of the closed loop response are upper bounded by
\begin{align}
(1-\rho(A+BK_{\star}))^{-1}&\le \poly\Big(\big(\frac{M}{\mu}\big)^{ \kappa},M^\kappa,\kappa\Big) \label{CTRL_eq:spectral_radius_margin}\\
\snorm{(zI-A-BK_{\star})^{-1}}_{\mathcal{H}_{\infty}}&\le \poly\Big(\big(\frac{M}{\mu}\big)^{ \kappa},M^\kappa,\kappa\Big)\label{CTRL_eq:hinfinity_margin}
\end{align}
\end{lemma}
\begin{proof}
First, note that since $Q=I$, immediately $(A,Q^{1/2})$ is observable and the stabilizing solution $P$ is well-defined. Note that the Riccati solution $P$ also satisfies the Lyapunov equation
\[
P=(A+BK_{\star})'P(A+BK_{\star})+I+K_{\star}'K_{\star}\succeq (A+BK_{\star})'P(A+BK_{\star})+I\succeq I.
\]
As a result,
\begin{equation}\label{CTRL_eq:Riccati_Inequality}
(A+BK_{\star})'(A+BK_{\star})\stackrel{i)}{\preceq} (A+BK_{\star})'P(A+BK_{\star})=P-I \stackrel{ii)}{\preceq} (1-\snorm{P}^{-1}_2)P,
\end{equation}
where i) follows from $P\succeq I$. To prove ii) observe that $P-I=P^{1/2}(I-P^{-1})P^{1/2}$ and $P^{-1}\succeq \snorm{P}^{-1}_2 I$.
Hence
\[
P-I\succeq P^{1/2}(I-\snorm{P}^{-1}_2 I)P^{1/2}=(1-\snorm{P}^{-1}_2)P.
\]
Applying inequality~\eqref{CTRL_eq:Riccati_Inequality} recursively
\[
(A+BK_{\star})^{t'}(A+BK_{\star})^t=\snorm{(A+BK_{\star})^t}^2_2 \le \big(1-\snorm{P}^{-1}_2\big )^{t} P.
\]
From here, we immediately deduce that
\[
\rho(A+BK_{\star})\le \sqrt{1-\snorm{P}^{-1}_2},
\]
which by Lemma~\ref{CTRL_lem:Riccati_Upper} proves~\eqref{CTRL_eq:spectral_radius_margin}.
For the $\mathcal{H}_{\infty}$ norm bound
\begin{align*}
\snorm{(zI-A-BK_{\star})^{-1}}_{\mathcal{H}_{\infty}}&\le \sum_{t\ge 0} \snorm{(A+BK_{\star})^t}_2\le \snorm{P}_2^{1/2}\frac{1}{1-\sqrt{1-\snorm{P}^{-1}_2}}\\
&\le \snorm{P}_2^{1/2}\frac{1+\sqrt{1-\snorm{P}^{-1}_2}}{\snorm{P}^{-1}_2}\le 2\snorm{P}_2^{3/2}.
\end{align*}
The proof of~\eqref{CTRL_eq:hinfinity_margin} now follows from Lemma~\ref{CTRL_lem:Riccati_Upper}.
\end{proof}

\section{Lower Bounds for the problem of Stabilization}\label{CTRL_app_sec:STAB_lower_bounds}
In this section, we prove Theorem~\ref{CTRL_thm:STAB_lower_exponential} by using information theoretic methods.
The main idea is to find systems that are nearly indistinguishable from data but require completely different stabilization schemes. 
We rely on Birgé's inequality~\citep{boucheron2013concentration}, which we review below for convenience. 

\begin{definition}[KL divergence]
Let $\P$, $\mathbb{Q}$ be two probability measures on some space $(\Omega,\mathcal{A})$. Let $\mathbb{Q}$ be absolutely continuous with respect to $\mathbb{P}$, that is $\mathbb{Q}(A)=\E_{\mathbb{P}}(Y 1_{A})$ for some integrable non-negative random variable with $\E_{\mathbb{P}}(Y)=1.$ The KL divergence $D(\mathbb{Q}||\P)$ is given by
\[
D(\mathbb{Q}||\P)\triangleq \E_{\mathbb{Q}}(\log Y).
\]
\end{definition}

\begin{theorem}[Birgé's Inequality~\citep{boucheron2013concentration}]\label{CTRL_thm:Birge}
Let $\P_0,\,\P_1$ be probability measures on $(\Omega,\mathcal{E})$ and let $E_0,\,E_1\in\mathcal{E}$ be disjoint events. If $1-\delta\triangleq \min_{i=0,1}\P_i(E_i)\ge 1/2$ then
\[
(1-\delta)\log\frac{1-\delta}{\delta}+\delta \log \frac{\delta}{1-\delta}\le D(\P_1||\P_0).
\]
\end{theorem}
The KL divergence between two Gaussian distributions with same variance is given below.
\begin{lemma}[Gaussian KL divergence]\label{CTRL_app_lem:Gaussian_KL}
Let $\P=\mathcal{N}(\mu_1,\sigma^2)$ and $\mathbb{Q}=\mathcal{N}(\mu_2,\sigma^2)$ then
\[
D(\mathbb{Q}||\P)=\frac{1}{2\sigma^2}(\mu_1-\mu_2)^2.
\]
\end{lemma}

\subsection{Proof of Theorem~\ref{CTRL_thm:STAB_lower_exponential}}
It is sufficient to prove it for $\kappa=n$. The proof for $\kappa<n$ is similar. Let $\alpha>0$ be such that $\alpha+\mu<1$. Consider the systems:
\[
S_1:\quad 	x_{k+1}=\matr{{ccccc}1 &\mu&0&\cdots&0\\0& \alpha&\mu&\cdots&0\\& &\ddots &\ddots&\\0&0&0&\cdots&\mu\\0&0&0&\cdots&\alpha}x_k+\matr{{c}0\\0\\\vdots\\0\\\mu}u_k+ \matr{{c}1\\0\\\vdots\\0\\0}w_k,
\]
\[
S_2:\quad x_{k+1}=\matr{{ccccc}1 &-\mu&0&\cdots&0\\0& \alpha&\mu&\cdots&0\\& &\ddots &\ddots&\\0&0&0&\cdots&\mu\\0&0&0&\cdots&\alpha}x_k+\matr{{c}0\\0\\\vdots\\0\\\mu}u_k+ \matr{{c}1\\0\\\vdots\\0\\0}w_k.
\]
By construction, the systems are $\mu-$robustly coupled.
Denote the state matrices by $A_1,A_2$ for $S_1,S_2$ respectively.
Let $\phi_1(z)=\det(zI-A_1-B\hat{K}_N)$, $\phi_2(z)=\det(zI-A_2-B\hat{K}_N)$ be the respective characteristic polynomials. By Jury's criterion~\citep[Ch. 4.5]{fadali2013digital}, a necessary (but not sufficient) condition for stability is:
\[
\phi_1(1)>0,\,\phi_2(1)>0.
\]
An direct computation gives:
\[
\phi_1(1)=\abs{\begin{array}{ccccc}0 &-\mu&0&\cdots&0\\0& 1-\alpha&-\mu&\cdots&0\\& &\ddots &\ddots&\\0&0&0&\cdots&-\mu\\-\hat{K}_{N,1}&-\hat{K}_{N,2}&-\hat{K}_{N,3}&\cdots&1-\alpha-\hat{K}_{N,n} \end{array}}=-\hat{K}_{N,1}\mu^{n-1},\, \phi_2(1)=\hat{K}_{N,1}\mu^{n-1}.
\]
As a result, the events:
\[
E_1=\set{\rho(A_1+B\hat{K}_{N})<1}\subseteq\set{\hat{K}_{N,1}<0},\quad E_2=\set{\rho(A_2+B\hat{K}_{N})<1}\subseteq\set{\hat{K}_{N,1}>0}
\]
are disjoint.  By Theorem~\ref{CTRL_thm:Birge}, a necessary condition for stabilizing both systems with probability larger than $1-\delta$ is:
\begin{equation}\label{CTRL_eq:STAB_necessary}
D(\mathbb{P}_{1}||\mathbb{P}_{2})\ge (1-2\delta)\log\frac{1-\delta}{\delta}\ge \log\frac{1}{2.4 \delta}\ge \log\frac{1}{3 \delta}.
\end{equation}
Here $\mathbb{P}_i$ is a shorthand notation for $\P_{S_i,\pi}$, for $i=1,2$.

Meanwhile, by the chain rule of KL divergence (see Exercise 4.4 in~\cite{boucheron2013concentration}):
\begin{align*}
D(\mathbb{P}_1||\mathbb{P}_2)&=\mathbb{E}_{\mathbb{P}_1}\Big( D(\mathbb{P}_1(\AUX)||\mathbb{P}_2(\AUX))\\&+\sum^{N}_{k=0}D(\mathbb{P}_1(x_k|x_{0:k-1},u_{0:k-1},\AUX)||\mathbb{P}_2(x_k|x_{0:k-1},u_{0:k-1},\AUX))\\
&+\sum^{N-1}_{k=0}D(\mathbb{P}_1(u_k|x_{0:k},u_{0:k-1},\AUX)||\mathbb{P}_2(u_k|x_{0:k},u_{0:k-1},\AUX)\Big),
\end{align*}
where $x_{0:k}$ is a shorthand notation for $x_{0},\dots,x_k$ (same for $u_{0:k}$). By $\mathbb{P}(X|Y)$ we denote the conditional distribution of $X$ given $Y$. Note that the inputs have the same conditional distributions under both measures hence their KL divergence is zero. As a result
\begin{align*}
D(\mathbb{P}_1||\mathbb{P}_2)&=\mathbb{E}_{\mathbb{P}_1}\sum^{N}_{k=0}D(\mathbb{P}_1(x_k|x_{0:k-1},u_{0:k-1},\AUX)||\mathbb{P}_2(x_k|x_{0:k-1},u_{0:k-1},\AUX))\\
&\stackrel{1)}{=}\mathbb{E}_{\mathbb{P}_1}
\sum^{N}_{k=0}D(\mathbb{P}_1(x_{k}|x_{k-1},u_{k-1})||\mathbb{P}_2(x_k|x_{k-1},u_{k-1})\\
&\stackrel{2)}{=}\mathbb{E}_{\mathbb{P}_1}
\sum^{N}_{k=0}D(\mathbb{P}_1(x_{k,1}|x_{k-1,1},x_{k-1,2})||\mathbb{P}_2(x_{k,1}|x_{k-1,1},x_{k-1,2})
\Big),
\end{align*}
where $1)$ follows from the Markov property of the linear system and 2) follows from an application of the chain rule, the structure of the dynamics, and the fact that all $x_{k,j}$ have the same distribution for $j\ge 2$. Recall that the normal distribution is denoted by $\mathcal{N}(\mu,\Sigma)$.
Now we can explicitly compute the KL divergence:
\begin{align}
D(\mathbb{P}_1||\mathbb{P}_2)&=\mathbb{E}_{\mathbb{P}_1}\sum^{N}_{k=1}D(\mathcal{N}(\alpha x_{k-1,1}+\mu x_{k-1,2},1)||\mathcal{N}(\alpha x_{k-1,1}-\mu x_{k-1,2},1))\nonumber\\
&\stackrel{i)}{=}\mathbb{E}_{\mathbb{P}_1}\sum_{k=1}^{N}2\mu^2 x^2_{k-1,2}=2\mu^2 \sum_{k=1}^{N}\mathbb{E}_{\mathbb{P}_1} x^2_{k-1,2}\label{CTRL_eq:KL_explicit_STAB},
\end{align}
where $i)$ follows by Lemma~\ref{CTRL_app_lem:Gaussian_KL}.
By~\eqref{CTRL_eq:STAB_necessary},~\eqref{CTRL_eq:KL_explicit_STAB}, and Lemma~\ref{CTRL_lem:S1_Gramian_like_bound_STAB}, it is necessary to have
\[
N\sigma^2_u  \ge \frac{1}{2}\paren{\frac{1}{\alpha+\mu}}^{2n-2}\paren{\frac{1-a-\mu}{\mu}}^{2}\log\frac{1}{3 \delta}
\]
Since we are free to choose $\alpha$, it is sufficient to choose $\alpha=0$. \hfill $\blacksquare$

\begin{lemma}\label{CTRL_lem:S1_Gramian_like_bound_STAB}
Consider system $S_1$ as defined above. Recall that $\P_1$ is a shorthand notation for $\P_{S_1,\pi}$. Then, under Assumption~\ref{CTRL_ass:input_budget}, we have
\[
\mathbb{E}_{\P_1} x^2_{k,2}\le \sigma^2_u (\alpha+\mu)^{2n-2} \paren{\frac{1}{1-(a+\mu)}}^2
\]
\end{lemma}
\begin{proof}
Let $e_2$ denote the canonical vector $e_2=\matr{{ccccc}0&1&0&\cdots&0}'$. Then
\[
 x_{k,2}=\sum_{t=1}^{k}e'_2A^{t-1}Bu_{k-t}=\sum_{t=n-1}^{k}e'_2A^{t-1}Bu_{k-t},
\]
where the second equality follows from the fact that $e'_2A^{t-1}B$, for $t\le n-1$. Moreover, we can upper bound:
\[
\abs{e'_2A^{t-1}B}\le (\alpha+\mu)^{t-1},
\]
which follows from the fact that the sub-matrix $[A_1]_{2:n,2:n}$ of $A_1$ if we delete the first row and column is bi-diagonal and Toeplitz hence $\snorm{[A_1]_{2:n,2:n}}_2\le \alpha+\mu$. Define $c_t\triangleq (\alpha+\mu)^{t-1}$. Then, we can upper bound $\abs{x_{k,2}}$ by
\[
\abs{x_{k,2}}\le \sum_{t=n-1}^{k}c_t \abs{u_{k-t}}.
\]
By Cauchy-Schwartz and Assumption~\ref{CTRL_ass:input_budget}
\[
\E_{S_1,\pi} u^2_{k}\le \sigma^2_u,\quad \E_{S_1,\pi} \abs{u_{k}u_t} \le \sigma^2_u.
\]
Finally, combining the above results
\[
\mathbb{E}_{S_1,\pi} x^2_{k,2}\le \sigma^2_u (\sum_{t=n-1}^k c_t)^2 \le \sigma^2_u (\alpha+\mu)^{2n-2} \paren{\frac{1}{1-(a+\mu)}}^2,
\]
which completes the proof.
\end{proof}

\section{Upper Bounds for the problem of Stabilization}\label{CTRL_app_sec:STAB_upper_bounds}
We employ a naive passive learning algorithm, where we employ a white-noise exploration policy to excite the state. Our gain design proceeds in two parts. First, we perform system identification based on least squares~\citep{simchowitz2018learning}. Second, we use robust control to design the gain  based on the identified model and bounds on the identification error of $A$ and $B$, similar to~\cite{dean2017sample}. 

\subsection{Algorithm}
 \begin{figure}
			\centerline{					\begin{tikzpicture}[auto,>=latex']
				\node [block] (plant) [draw, align=center] {White Noise\\Experiments};
				\node [block, right=2.5cm of plant] (SID) [draw, align=center] {System\\Identification};
				\node [block, right=1.8cm of SID] (controller) [draw, align=center] {Controller\\Design};
				\coordinate [right=1.5cm of controller] (help_1) {};   
				\draw [->] (plant) -- node[name=y,above] {$x_0,\dots,x_N$} node[name=y,below] {$u_0,\dots,u_{N-1}$}(SID);
				\draw [->] (SID) -- node[name=y,above] { $\hat{A}_N$, $\hat{B}_N$} node[name=y,below] { $\epsilon_A$, $\epsilon_B$} (controller);
				\draw [->] (controller) --
				node[name=y,above] {$\hat{K}_N$}
				(help_1);	
				\end{tikzpicture}
			}
			\caption{The block diagram of the stabilization scheme. First, we generate white noise inputs $u_t\sim\mathcal{N}(0,\bar{\sigma}^2_uI)$ to excite the system. Then we perform system identification based on least squares to obtain estimates $\hat{A}_N,\hat{B}_N$ of the true system matrices. Finally, we design a controller gain $\hat{K}_N$, based on the system estimates and upper bounds $\epsilon_A,\epsilon_B$ on the estimation error.}
			\label{CTRL_fig:stabilization_architecture}
			\end{figure}
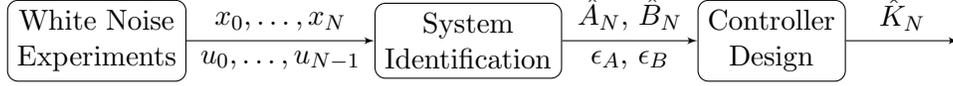
The block diagram for the algorithm is shown in~Fig.~\ref{CTRL_fig:stabilization_architecture}. To generate the input data $u_0,\dots,u_{N-1}$, we employ white noise inputs $u_k\sim\mathcal{N}(0,\bar{\sigma}^2_u I)$, $\bar{\sigma}^2_u=\sigma^2_u/p$, where we normalize with $p$ in order to satisfy Assumption~\ref{CTRL_ass:input_budget}. For the system identification part, we use a least squares algorithm
\begin{equation}\label{CTRL_eq:least_squares}
\matr{{cc}\hat{A}_N&\hat{B}_N}=\arg\min_{\set{F\in\R^{n\times n},G\in \R^{n\times p}}} \sum_{t=0}^{N-1}\snorm{x_{t+1}-Fx_t-Gu_t}^2_2,
\end{equation}
to obtain estimates of the matrices $A\,,B$. Now, let $\epsilon_A,\,\epsilon_B$ be large enough constants such that $\snorm{A-\hat{A}_N}_2\le \epsilon_A$, $\snorm{B-\hat{B}_N}_2\le \epsilon_B$. To design the controller gain $\hat{K}_N$, it is sufficient to solve the following problem
\begin{equation}
\begin{aligned}\label{CTRL_eq:stabilization_scheme}
    \mathrm{find}\:&\quad  {K\in\R^{p\times n}}\\
    \mathrm{s.t.}
    &\quad \norm{\matr{{c}\sqrt{2}\epsilon_A (zI-\hat{A}_N-\hat{B}_NK)^{-1}\\\sqrt{2}\epsilon_B K(zI-\hat{A}_N-\hat{B}_NK)^{-1}}}_{\mathcal{H}_{\infty}}<1.
\end{aligned}
\end{equation}
The idea behind the scheme is the following. Let $\hat{K}_N$ be a gain that stabilizes the estimated plant $(\hat{A}_N,\hat{B}_N)$. To make sure that it also stabilizes the nominal plant $(A,B)$ we impose some additional robustness conditions. In fact, as we show in Theorem~\ref{CTRL_thm:feasibility_STAB}, any feasible gain of problem~\eqref{CTRL_eq:stabilization_scheme} will stabilize any plant $(\hat A,\hat B)$ that satisfies $\snorm{\hat A-\hat{A}_N}_2\le \epsilon_A$, $\snorm{\hat B-\hat{B}_N}_2\le \epsilon_B$, including the nominal one.
In this work, we do not study how to efficiently solve~\eqref{CTRL_eq:stabilization_scheme}. For efficient implementations one can refer to~\cite{dean2017sample}. Note that the certainty equivalent LQR design~\citep{mania2019certainty} or the SDP relaxation method~\citep{cohen2018online,chen2021black} could also work as stabilization schemes.

\subsection{System Identification Analysis}
Here we review a fundamental system identification result from~\cite{simchowitz2018learning}. The original proof can be easily adapted to the case of singular noise matrices $H$~\citep{tsiamis2021linear}.
\begin{theorem}[Identification Sample Complexity]\label{CTRL_thm:identification}
Consider a system $S=(A,B,H)$ such that Assumption~\ref{CTRL_ass:general_setting} is satisfied. Let $(A,B)$ be controllable with $\Gamma_{k}=\Gamma_k(A,B)$ the respective controllability Gramian and $\kappa=\kappa(A,B)$ the respective controllability index.  Then, under the least squares system identification algorithm~\eqref{CTRL_eq:least_squares} and white noise inputs $u_k~\sim\N(0,\bar{\sigma}^2_u I_p)$, we obtain
\begin{align}
		\P_{S,\pi}(\snorm{\matr{{cc}A-\hat{A}_N&B-\hat{B}_N}}_2\ge \epsilon)\le \delta \nonumber
	\end{align}
if we have a large enough sample size
	\[
	N\bar{\sigma}^2_u\ge    \frac{\poly(n,\log1/\delta,M)}{\epsilon^2\sigma_{\min}(\Gamma_{\kappa})}\log N.
	\]
\end{theorem}
\begin{proof}
The proof is almost identical to the one of Theorem~4 in~\cite{tsiamis2021linear}. The difference is that here we consider only the Gramian and index of $(A,B)$ in the final bound, while in~\cite{tsiamis2021linear} the Gramian and index of $(A\matr{{cc}H&B})$ appears. We repeat the proof here to avoid notation ambiguity. Our goal is to apply Theorem~2.4 in~\citep{simchowitz2018learning}. Define the noise-controllability Gramian $\Gamma^h_{t}=\Gamma_t(A,H)$ as well as the combined controllability Gramian
\[
\Gamma^c_t=\Gamma_t(A,\matr{{cc}\bar{\sigma}_u B&H})=\bar{\sigma}^2_u\Gamma_{t}+\Gamma^h_{t}.
\]
Define $y_{k}=\matr{{cc}x'_k&u'_k}'$. It follows that for all $j\ge 0$ and all unit vectors $v\in\R^{(n+p)\times 1},$ the following small-ball condition is satisfied:
\begin{equation}\label{CTRL_eq:small_ball}
    \frac{1}{2\kappa}\sum_{t=0}^{2\kappa}\P(\abs{v'y_{t+j}}\ge \sqrt{v'\Gamma_{\mathrm{sb}}v} |\bar{\F}_j)\ge \frac{3}{20},
\end{equation}
where
\begin{equation}\label{CTRL_eq:small_ball_covariance}
    \Gamma_{\mathrm{sb}}=\matr{{cc}\Gamma^c_{\kappa}&0\\0&\bar{\sigma}^2_uI_p}.
\end{equation}
Equation~\eqref{CTRL_eq:small_ball} follows from the same steps as in Proposition~3.1 in~\cite{simchowitz2018learning} with the choice $k=2\kappa$.

Next, we determine an upper bound $\bar{\Gamma}$ for the gram matrix $\sum_{t=0}^{N-1}y_ty'_t$. Using a Markov inequality argument as in~\cite[proof of Th 2.1]{simchowitz2018learning}, we obtain that
\[
\P(\sum_{t=0}^{N-1}y_ty'_t\preceq \bar{\Gamma})\ge 1-\delta,
\]
where 
\[\bar{\Gamma}=\frac{n+p}{\delta}N\matr{{cc}\Gamma^c_N&0\\0&\bar{\sigma}^2_u I_p}.\]

Now, we can apply Theorem 2.4 of~\cite{simchowitz2018learning}. With probability at least $1-3\delta$ we have $\snorm{\matr{{cc}A-\hat{A}_N&B-\hat{B}_N}}_2\le \epsilon$ if:
\begin{align*}
   N&\ge \frac{\poly(n,\log1/\delta,M)}{\epsilon^2\sigma_{\min}(\Gamma^c_{\kappa})}\log\det (\bar{\Gamma}\Gamma^{-1}_{\mathrm{sb}}),
\end{align*}
where we have simplified the expression by including terms in the polynomial term.
Using Lemma~1 in~\cite{tsiamis2021linear}, we obtain
\[
\log\det (\bar{\Gamma}\Gamma^{-1}_{\mathrm{sb}})=\poly(n,M,\log 1/\delta)\log N.
\]
Moreover, we use the lower bound $\Gamma^c_k\succeq \bar{\sigma}^2_u\Gamma_k$, which holds for every $k\ge 0$.
\end{proof}
We note that we can easily obtain sharper bounds by considering the combined controllability Gramian $\Gamma_k(A,\matr{{cc}\bar{\sigma}_u B&H})$ for the identification stage. For the economy of the presentation, we omit such an analysis here.

\subsection{Sensitivity of Stabilization}
Here we prove that when~\eqref{CTRL_eq:stabilization_scheme} is feasible, then $\hat{K}_N$ stabilizes all plants $(A,B)$ such that $\snorm{A-\hat{A}_N}_2\le \epsilon_A$, $\snorm{B-\hat{B}_N}_2\le \epsilon_B$. We also show that feasibility is guaranteed as long as we can achieve small enough error bounds $\epsilon_A$, $\epsilon_B$.
\begin{theorem}\label{CTRL_thm:feasibility_STAB}
Let $\hat{K}_N$ be a feasible solution to problem~\eqref{CTRL_eq:stabilization_scheme} for some $\epsilon_A,\epsilon_B>0$. Then for any system $(A,B)$ such that $\snorm{A-\hat{A}_N}_2\le \epsilon_A$, $\snorm{B-\hat{B}_N}_2\le \epsilon_B$ we have that
\[
\rho(A+B\hat{K}_N)<1.
\]
Moreover, there exists an $\epsilon_0>0$ such that
\[
\epsilon_0=\poly\Big(\big(\frac{M}{\mu}\big)^{ \kappa},M^\kappa,\kappa\Big)
\]
and Problem~\eqref{CTRL_eq:stabilization_scheme} is feasible if $\epsilon_A,\epsilon_B\le \epsilon_0$.
\end{theorem}
\begin{proof}
Let $\hat{K}_N$ be a feasible solution to problem~\eqref{CTRL_eq:stabilization_scheme}.
Define $\mbf{\Phi}_x=(zI-\hat{A}_N-\hat{B}_N\hat{K}_N)^{-1}$, which is well-defined and stable since $\epsilon_A>0$ and $\snorm{\mbf{\Phi}_x}_{\mathcal{H}_{\infty}}<1/(\sqrt{2}\epsilon_A)$. Define the system difference
\[
\mbf{\Delta}\triangleq (\hat{A}_N-A)\mbf{\Phi}_x+(\hat{B}_N-B)\hat{K}_N\mbf{\Phi}_x
\]
It follows from simple algebra that:
\begin{align*}
zI-A-B\hat{K}_N&=zI-\hat{A}_N-\hat{B}_N\hat{K}_N+(\hat{A}_N-A)+(\hat{B}_N-B)\hat{K}_N\\
&=(I+\mbf{\Delta})(zI-\hat{A}_N-\hat{B}_N\hat{K}_N).
\end{align*}
If $(I+\mbf{\Delta})^{-1}$ is stable then the closed loop response is stable and well-defined
\[
(zI-A-B\hat{K}_N)^{-1}=(zI-\hat{A}_N-\hat{B}_N\hat{K}_N)^{-1}(I+\mbf{\Delta})^{-1}.
\]
But $(I+\mbf{\Delta})^{-1}$ being stable is equivalent to
\[
\snorm{(I+\mbf{\Delta})^{-1}}_{\mathcal{H}_{\infty}}<\infty.
\]
A sufficient condition for this to occur is to require~\citep{dean2017sample}
\[
\snorm{\mbf{\Delta}}_{\mathcal{H}_{\infty}}<1.
\]
By Proposition~3.5 (select $\alpha=1/2$) of~\citep{dean2017sample}
\[
\snorm{\mbf{\Delta}}_{\mathcal{H}_{\infty}}<\norm{\matr{{c}\sqrt{2}\epsilon_A (zI-\hat{A}_N-\hat{B}_NK)^{-1}\\\sqrt{2}\epsilon_B K(zI-\hat{A}_N-\hat{B}_NK)^{-1}}}_{\mathcal{H}_{\infty}}<1.
\]
This completes the proof of $\rho(A+B\hat{K}_N)<1$.

To prove feasibility consider the optimal LQR gain $K_{\star}$, for $Q=I_n$, $R=I_p$. Following Lemma~4.2 in~\cite{dean2017sample}, if the following sufficient condition holds
\[
(\epsilon_{A}+\epsilon_B\snorm{K_{\star}}_2)\snorm{(zI-A-BK_{\star})^{-1}}_{\mathcal{H}_{\infty}}\le 1/5,
\]
then $K_{\star}$ is a feasible solution
\[
\norm{\matr{{c}\sqrt{2}\epsilon_A (zI-\hat{A}_N-\hat{B}_NK_{\star})^{-1}\\\sqrt{2}\epsilon_B K_{\star}(zI-\hat{A}_N-\hat{B}_NK_{\star})^{-1}}}_{\mathcal{H}_{\infty}}<1.
\]
Hence, we can choose
\begin{equation}\label{CTRL_eq:epsilon_feasibility_STAB}
\epsilon_0=\big(5(1+\snorm{K_{\star}}_2)\snorm{(zI-A-BK_{\star})^{-1}}_{\mathcal{H}_{\infty}}\big)^{-1}.
\end{equation}
The fact that $\epsilon_0=\poly\Big(\big(\frac{M}{\mu}\big)^{ \kappa},M^\kappa,\kappa\Big)$ follows from Lemmas~\ref{CTRL_lem:Riccati_Upper},~\ref{CTRL_lem:margins}.
\end{proof}

\subsection{Proof of Theorem~\ref{CTRL_thm:upper_bounds_STAB}}
 Let $u_t\sim\mathcal{N}(0,\bar{\sigma}^2_u I)$, with $\bar{\sigma}^2_u=\sigma^2_u/p$. Consider the stabilization algorithm as described in~\eqref{CTRL_eq:least_squares},~\eqref{CTRL_eq:stabilization_scheme}. Consider the $\epsilon_0$ defined in~\eqref{CTRL_eq:epsilon_feasibility_STAB}. By Theorems~\ref{CTRL_thm:identification},~\ref{CTRL_thm:feasibility_STAB}, if
\[
N\sigma^2_{u}\ge \triangleq \underbrace{\frac{\poly(n,\log 1/\delta,M)}{\epsilon^2_0\sigma_{\min}(\Gamma_{\kappa})}}_{\mathcal{N}}\log N
\]
we have with probability at least $1-\delta$ that $\snorm{A-\hat{A}_N}_2,\snorm{B-\hat{B}_N}_2\le \epsilon_0$ and problem~\eqref{CTRL_eq:stabilization_scheme} is feasible with $\epsilon_B=\epsilon_A=\epsilon_0$. By Theorems~\ref{CTRL_thm:gramian_lower_bound}~\ref{CTRL_thm:feasibility_STAB},
\[
\mathcal{N}=\poly\paren{\Big(\frac{M}{\mu}\Big)^\kappa,M^{\kappa},n,\log 1/\delta}.
\]
To complete the proof we use the fact that
\[
N\ge c\log N\text{ if }N\ge 2c\log 2c.
\]
\section{Regret Lower Bounds}
First let us state an application of the main result of~\cite{ziemann2022regret}. Consider a system $(A,B,H)\in\R^{n\times (n+p+n)}$, where $(A,B)$ is controllable and $H=I_n$. Let $P$ be the respective Riccati matrix for $Q=I_n$, $R=I_p$, with $K_{\star}$ the respective optimal LQR gain. Fix a matrix $\Delta\in\R^{p\times n}$ and define the family of systems:
\begin{equation}\label{CTRL_eq:general_system_family}
A(\theta)=A-\theta B\Delta,\, B(\theta)=B+\theta \Delta,\, H(\theta)=I_n,
\end{equation}
where $\theta\in \mathcal{B}(0,\epsilon)$, for some small $\epsilon$.
Assume that $\epsilon$ is small enough, such that the Riccati equation has a stabilizing solution for every system in the above family. The respective Riccati matrix is denoted by $P(\theta)$ and the LQR gain by $K(\theta)$. The derivative of $K_{\star}(\theta)$ with respect to $\theta$ at point $\theta=0$ is given by the following formula.
\begin{lemma}[Lemma 2.1~\citep{simchowitz2020naive}]\label{CTRL_lem:K_derivative}
If the system $(A,B)$ is stabilizable, then
\[
\frac{d}{d\theta}K_{\star}(\theta)|_{\theta=0}=-(B'PB+R)^{-1}\Delta'P(A+BK_*).
\]
\end{lemma}
Finally, let $\Sigma_x$ be the solution to the Lyapunov equation:
\begin{equation}\label{CTRL_eq:appendix_steady_state_covariance}
\Sigma_x=(A+BK_{\star})\Sigma_x(A+BK_{\star})'+I_n.
\end{equation}

\begin{theorem}[Application of Theorem~1 in~\cite{ziemann2022regret}]\label{CTRL_thm:variation_lower_bounds}
Consider a system $S=(A,B,H)\in\R^{n\times (n+p+n)}$, where $(A,B)$ is controllable and $H=I_n$. Let $P$ be the respective solution of the algebraic Riccati equation for $Q=I_n$, $R=I_p$, with $K_{\star}$ the respective optimal LQR gain. Recall the definition of $\Sigma_x$ in~\eqref{CTRL_eq:appendix_steady_state_covariance}. Define the family of systems $\CC_S({\epsilon})\triangleq\set{(A(\theta),B(\theta),I_n),\,\theta\in\mathcal{B}(0,\epsilon)}$ as defined in~\eqref{CTRL_eq:general_system_family}, for any $\epsilon>0$ sufficiently small such that $P(\theta)$ and $K_{\star}(\theta)$ are well-defined. Let $Q_T=P(\theta)$.  Then for any $\alpha\in (0,1/4)$:
\begin{align}\label{CTRL_eq:main_technical_lower_bound}
&\lim\inf_{T\rightarrow \infty}\sup_{\hat{S}\in\CC_S(T^{-a})}\E_{\hat{S},\pi}\frac{R_{T}(\hat{S})}{\sqrt{T}}\ge \frac{1}{2\sqrt{2}}\sqrt{\frac{F}{L}},
\end{align}
where 
\begin{align*}
F&=\Tr\bigg((B'PB+R)^{-1}\Delta' P \clint{\Sigma_x-I_n}P\Delta\bigg)\\
L&= n (\snorm{\Delta K_{\star}}^2_2+\snorm{\Delta}^2_2)\snorm{(B'PB+R)^{-1}}_{2}
\end{align*}
\end{theorem}
\begin{proof}
Note that if $\Delta'P(A+BK_{\star})=0$, then since $\Sigma_x\succeq I_n$ is invertible
\begin{align*}
\Delta'P(A+BK_{\star})=0&\Leftrightarrow \Delta'P(A+BK_{\star})\Sigma_x(A+BK_{\star})'P\Delta=0\\
&\Leftrightarrow \Delta'P(\Sigma_x-I_n)P\Delta=0.
\end{align*}
This implies that $F=0$ and the regret lower bound becomes $0$, in which case the claim of the theorem is trivially true.
Hence, we will assume that $\Delta'P(A+BK_{\star})\neq 0$.

All systems in the family have the same closed-loop response under the control policy $u=K_{\star}x$. In particular, for all $\theta\in \mathcal{B}(0,\epsilon)$:
\[
\frac{d}{d\theta}\matr{{cc}A(\theta)&B(\theta)}\matr{{c}I_n\\K_{\star}}=\matr{{cc}-\Delta K_{\star}&\Delta}\matr{{c}I_n\\K_{\star}}=0.
\]
Moreover, by Lemma~\ref{CTRL_lem:K_derivative}
\[
\frac{d}{d\theta}K_{\star}(\theta)|_{\theta=0}=(B'PB+R)^{-1}\Delta'P(A+BK_{\star})\neq 0.
\]
By Proposition 3.4 in~\cite{ziemann2022regret}, the above two conditions imply that the family $\CC_{S}(\epsilon)$ is $\epsilon-$uninformative (see Section 3 in~\cite{ziemann2022regret} for definition). 

Next, by Lemma 3.6 in~\cite{ziemann2022regret}, the family is also $L-$information regret bounded (see Section 3 in~\cite{ziemann2022regret} for the definition), where
\[
L=\Tr(I_n)\snorm{\matr{{cc}-\Delta K_{\star}&\Delta}}^2_2\snorm{(B'PB+R)^{-1}}_{2}\stackrel{i)}{\le} n (\snorm{\Delta K_{\star}}^2_2+\snorm{\Delta}^2_2)\snorm{(B'PB+R)^{-1}}_{2}.
\]
Inequality $i)$ follows from $\Tr(I_n)=n$ and the norm property \[\snorm{\matr{{cc}M_1&M_2}}^2_2=\snorm{\matr{{cc}M_1&M_2}\matr{{cc}M_1&M_2}'}_2=\snorm{M_1M_1'+M_2M_2'}_2 \le \snorm{M_1}^2_2+\snorm{M_2}^2_2.\]

Applying Theorem~1 in~\cite{ziemann2022regret}, we get~\eqref{CTRL_eq:main_technical_lower_bound}, for $L$ defined as above and
\[
F=\Tr\bigg(\clint{\Sigma_x \otimes (B'P(\theta)B+R)}(\frac{d}{d\theta}\VEC K_{\star}(\theta)|_{\theta=0})(\frac{d}{d\theta}\VEC K_{\star}(\theta)|_{\theta=0})'\bigg),
\]
where $\otimes$ is the Kronecker product and $\VEC$ is the vectorization operator (mapping a matrix into a column vector by stacking its
columns). Using the identities:
\[
\VEC(XYZ)=(Z'\otimes X)\VEC(Y),\qquad \Tr(\VEC(X)\VEC(Y)')=\Tr (XY'),
\]
we can rewrite $F$ as
\[
F=\Tr\bigg((B'P(\theta)B+R)\frac{d}{d\theta}K(\theta)|_{\theta=0}\Sigma_x \frac{d}{d\theta}K'(\theta)|_{\theta=0}\bigg).
\]
By Lemma~\ref{CTRL_lem:K_derivative} and the property $\Tr(XY)=\Tr(YX)$, we finally get
\[
F=\Tr\bigg((B'PB+R)^{-1} \Delta'P(A+BK_*)\Sigma_x(A+BK_*) P\Delta\bigg).
\]
The result follows from $(A+BK_*)\Sigma_x(A+BK_*)'=\Sigma_x-I_{n}$.
\end{proof}

\subsection{Proof of Lemma~\ref{CTRL_lem:modular_bound_two_subsystems}}
The result follows by Theorem~\ref{CTRL_thm:variation_lower_bounds}. We only need to compute and simplify $F$, $L$.
Due to the structure of system~\eqref{CTRL_eq:composite_system}, we have
\[
P=\matr{{cc}1&0\\0&P_0},\,K_{\star}=\matr{{cc}0&0\\0&K_{0,\star}}.
\]
Moreover, due to the structure of the perturbation $\Delta$ in~\eqref{CTRL_eq:perturbation_structure}
\[
B'PB+R=\matr{{cc}2&0\\0&B'_0P_0B_0+R_0},\, P\Delta(B'PB+R)^{-1}\Delta'P=\frac{1}{2}\matr{{cc}0&0\\0&P_0\Delta_1\Delta'_1P_0}.
\]
Hence
\[
F=\frac{1}{2}\Tr\bigg(\matr{{cc}0&0\\0&P_0\Delta_1\Delta'_1P_0}(\Sigma_{x}-I_n)\bigg)=\frac{1}{2}\Delta'_1P_0(\Sigma_{0,x}-I_{n-1})P_0\Delta_1
\]
Finally we have $L\le n$, since $\Delta K_{\star}=0$, $\Delta_1$ has unit norm, and $R=I_p$.\hfill $\blacksquare$

\subsection{Proof of Lemma~\ref{CTRL_lem:integrator_system_theoretic_parameters}}
First note that $P_0\succeq Q_0=I_{n-1}$. As a result, we have
\[
\snorm{P_0(\Sigma_{0,x}-I_{n-1})P_0}_2\ge \snorm{\Sigma_{0,x}-I_{n-1}}_2.
\]
It is sufficient to lower bound $\snorm{\Sigma_{0,x}-I_{n-1}}_2$.
Consider the recursion:
\[
\Sigma_k=(A_0+B_0K_{0,\star})\Sigma_{k-1}(A_0+B_0K_{0,\star})'+I_{n-1},\,\Sigma_{0}=0.
\]
Then $\Sigma_{0,x}=\lim_{k\rightarrow \infty}\Sigma_{k}\succeq \Sigma_{n-1}\succeq I_{n-1}$. The second inequality follows from monotonicity of the Lyapunov operator:
\[
g(X)=(A_0+B_0K_{0,\star})X(A_0+B_0K_{0,\star})'+I_{n-1},
\]
i.e. $g(X)\succeq g(Y) $ if $X\succeq Y$. What remains is to lower bound $\snorm{\Sigma_{n-1}-I_{n-1}}_2$. Let $e_1=\matr{{cccc}1&0&\cdots&0}'$ be the first canonical vector. Due to the structure of $A_0,B_0$
\[
e_1'(A_0+B_0K_{0,\star})^i=e_1'(A_0)^i,\text{ for }i\le n-1.
\]
Hence
\begin{align*}
\snorm{\Sigma_{n-1}-I_{n-1}}_2&\ge e_1'(\Sigma_{n-1}-I_{n-1})e_1\\
&=\sum_{k=1}^{n-1} e_1' A^k_0(A'_0)^k e_1.
\end{align*}
After some algebra we can compute analytically
\begin{align*}
\snorm{\Sigma_{n-1}-I_{n-1}}_2&\ge \sum_{k=1}^{n-1} \sum_{t=0}^{k}\binom{k}{t}^2=\sum_{k=1}^{n-1} \binom{2k}{k}\ge \binom{2(n-1)}{n-1}\ge \paren{2\frac{n-1}{n-1}}^{n-1}=2^{n-1},
\end{align*}
which completes the proof. \hfill $\blacksquare$

\subsection{Proof of Theorem~\ref{CTRL_thm:REG_lower_exponential}}
It is sufficient to prove the result for the class $\C^{\mu}_{n,n-1}$. If $n>\kappa+1$, then we can consider the system:
\[
\tilde{A}=\matr{{c|c}0&0\\\hline 0&A},\, \tilde{B}=\matr{{c|c}I_{n-\kappa-1}&0\\\hline 0&B},\, \tilde{H}=\matr{{c|c}I_{n-\kappa-1}&0\\\hline 0&H}
\]
where $(A,B,H)\in\CC^{\mu}_{\kappa,\kappa-1}$ and repeat the same arguments.

The proof follows from Lemma~\ref{CTRL_lem:modular_bound_two_subsystems} and Lemma~\ref{CTRL_lem:integrator_system_theoretic_parameters}. What remains to show that for every $\epsilon$
\[
\CC_{S}(\epsilon)\subseteq \CC^{\mu}_{n,n-1}(\epsilon).
\]
This follows from the fact that $\Delta K_{\star}=0$, hence $A=A(\theta)$ and $\snorm{B-B(\theta)}=\theta\snorm{\Delta}_2=\theta \le \epsilon$. Thus,
\[
\snorm{\matr{{cc}A-A(\theta)&B-B(\theta)}}_2\le \epsilon.
\]
Since $\CC_{S}(\epsilon)\subseteq \CC^{\mu}_{n,n-1}(\epsilon)$, we get
\[
\lim\inf_{T\rightarrow \infty}\sup_{S\in \CC^{\mu}_{n,n-1}({T^{-a}})}\E_{\hat{S},\pi}\frac{R_{T}(\hat{S})}{\sqrt{T}}\ge \lim\inf_{T\rightarrow \infty}\sup_{\hat{S}\in \CC_S({T^{-a}})}\E_{\hat{S},\pi}\frac{R_{T}(\hat{S})}{\sqrt{T}}
\tag*{$\blacksquare$}\]

\subsection{Stable System Example}\label{CTRL_sec:REG_stable}
Here we show that the local minimax expected regret can be exponential in the dimension even for stable systems. 
Using again the two subsystems trick, consider the following stable system
\begin{equation}\label{CTRL_eq:REG_difficult_example_stable}
S:\qquad x_{k+1}=\matr{{c|ccccc}0&0&0&&0&0\\\hline 0&\rho&2&&0&0\\& &&\ddots&\\0&0&0& &\rho&2\\0&0&0& &0&\rho}x_k+\matr{{c|c}1&0\\0&0\\\vdots\\0&1}u_k+w_{k},\,0<\rho<1,
\end{equation}
with $Q=I_n$, $R=I_2$.
Following the notation of~\eqref{CTRL_eq:composite_system} let:
\begin{equation}\label{CTRL_eq:REG_difficult_example_stable_subsystem}
A_0=\matr{{cccccc}\rho&2&0&&0&0\\0&\rho&2&&0&0\\& &&\ddots&\\0&0&0& &\rho&2\\0&0&0& &0&\rho},\, B_0=\matr{{c}0\\0\\\vdots\\0\\1},\,Q_0=I_{n-1},\,R_0=1,
\end{equation}
where $A_0\in\R^{(n-1)\times (n-1)}$ and $B_0\in \R^{n-1}$. Note that $A_0$ has spectral radius $\rho<1$.
Let $\Delta=\matr{{cc}0&0\\\Delta_1&0}$. Then, by Lemma~\ref{CTRL_lem:modular_bound_two_subsystems}, the local minimax expected regret for system $S$, given the perturbation $\Delta_1$ is lower bounded by
\begin{align*}
&\lim\inf_{T\rightarrow \infty}\sup_{\hat{S}\in \CC_S({T^{-a}})}\E_{\hat{S},\pi}\frac{R_{T}(\hat{S})}{\sqrt{T}}\ge \frac{1}{4\sqrt{n}}\sqrt{\Delta'_1 P_0 \clint{\Sigma_{0,x}-I_{n-1}}P_0\Delta_1}.
\end{align*}
As we show in the following lemma, the quantity $\sqrt{\Delta'_1 P_0 \clint{\Sigma_{0,x}-I_{n-1}}P_0\Delta_1}$ is exponential with $n$ if we choose $\Delta_1$ appropriately. Although the system is stable, it is very sensitive to inputs and noises. Any signal $u_{k,2}$ that we apply gets amplified by $2$ as we move up the chain from state $x_{k,n}$ to state $x_{k,2}$. As a result, any suboptimal policy will result in excessive excitation of the state.

\begin{lemma}[Stable systems can be hard to learn]\label{CTRL_lem:stable_system_theoretic_parameters}
Consider system~\eqref{CTRL_eq:REG_difficult_example_stable_subsystem}
Let $P_0$ be the Riccati matrix for $Q_0=I_{n-1},R_0=1$, with $K_{\star,0}$, $\Sigma_{0,x}$ the corresponding LQR control gain and steady-state covariance, respectively. Then 
\[\snorm{P_0 \clint{\Sigma_{0,x}-I_{n-1}}P_0}_2\ge 2^{4n-8}+o(1), \]
where $o(1)$ goes to zero as $n\rightarrow \infty$.
\end{lemma}
\begin{proof}
 Let $\Delta_1=\matr{{ccccc}0&0&\cdots&1&0}'$. It is sufficient to prove that
\[
\Delta_1'P_0(\Sigma_{0,x}-I_{n-1})P_0\Delta_1
\]
is exponential.
Using the identity $\Sigma_{0,x}-I_{n-1}=(A_0+B_0K_{\star,0})\Sigma_{0,x}(A_0+B_0K_{\star,0})'$, $\Sigma_{0,x}\succeq I$, we have:
\[
\Delta_1'P_0(\Sigma_{0,x}-I_{n-1})P_0\Delta_1 \ge \snorm{\Delta'_1 P_0 (A_0+B_0K_{\star,0})}^2_2.
\]
By Lemma~\ref{CTRL_lem_app:stable_aux1} and Lemma~\ref{CTRL_lem:exponential_riccati_stable} it follows that
\[
\snorm{\Delta'_1 P_0 (A_0+B_0K_{\star,0})}^2_2 \ge 2^{4n-8}+o(1).
\]
\end{proof}
\begin{lemma}[Riccati matrix can grow exponentially]\label{CTRL_lem:exponential_riccati_stable}
For system~\eqref{CTRL_eq:REG_difficult_example_stable_subsystem} we have:
\[
B_0'P_{0}B_0+R_0\ge 2^{2n-4}+1.
\]
\end{lemma}
\begin{proof} 
Consider the Riccati operator:
\[
g(X,Y)=A_0'XA_0+Y-A_0'XB_0(B_0'XB_0+R_0)^{-1}B_0'XA_0.
\]
Based on the above notation, we have $P_{0}=g(P_{0},Q_0)$.
The Riccati operator is monotone~\citep{anderson2005optimal}, i.e
\[
X_1\succeq X_2\Rightarrow g(X_1,Y)\succeq g(X_1,Y).
\]
It is also trivially monotone with respect to $Y$.
Let $X_0=0$, then the recursion $X_{t+1}=g(X_t,Q_0)$ converges to $P_{0}$. By monotonicity \[P_{0}\succeq X_t, \text{ for all }t\ge 0\]
Let $e_i$ denote the $i$-th canonical vector in $\R^{n-1}$. By monotonicity, we also have:
\[
X_1=g(X_0,Q_0)\succeq g(X_0,e_1e_1')=\underbrace{e_1e_1'}_{\tilde{X}_1}
\]
Repeating the argument:
\begin{align*}
X_2&=g(X_1,Q_0)\succeq g(\tilde{X}_1,Q_0)\succeq g(\tilde{X}_1,e_1e_1')=\underbrace{A_0'\tilde{X}_1A_0+e_1e_1'}_{\tilde{X}_2}=A_0'e_1e_1'A_0+e_1e_1'\\
&=2^2e_2e_2'+\rho^2 e_1e_1'+2\rho e_1e_2'+2\rho e_2e_1'
\end{align*}
Similarly,
\begin{align*}
X_{n-1}=g(X_{n-2},Q_0)\succeq g(\tilde{X}_{n-2},e_1e_1')= (A_0')^{n-2}e_1e_1'A_0^{n-2}+(A_0')^{n-1}e_1e_1'A_0^{n-1}+\dots+e_1e_1',
\end{align*}
where we use the fact that every $\tilde{X}_k$ is orthogonal to $B_0$ for $k\le n-2$. As a result:
\begin{align}
[P_{0}]_{n-1,n-1}&\ge [X_n]_{n-1,n-1}\ge e'_{n-1}(A_0')^{n-2}e_1e_1'A_0^{n-2}e_{n-1}\nonumber\\
&=(e_1'A_0^{n-2}e_{n-1})^2=([A_0^{n-2}]_{1,n-1})^2 \label{CTRL_app_eq:P_lower_bound_stable}
\end{align}
What remains is to compute $[A_0^{n-2}]_{1,n-1}$.
Define by $J\in\R^{(n-1)\times (n-1) }$ the companion matrix:
\[
J=\matr{{cccccc}0&1&0&&0&0\\0&0&1&&0&0\\& &&\ddots&\\0&0&0& &0&1\\0&0&0& &0&0}.
\]
Since $A_0=\rho I+2 J$ and $I$ commutes with $J$ by the binomial expansion formula:
\[
A_0^{n-2}=2^{n-2}J^{n-2}+\sum_{t=0}^{n-3}2^{t}\binom{n-2}{t}J^{t}.
\]
Since $e_1'J^{n-1}e_{n-1}=1$, $e_1'J^{t}e_{n-1}=0,$ for $t\le n-2$, we obtain:
\begin{equation}\label{CTRL_app_eq:P_lower_bound_stable_A}
([A_0^{n-2}]_{1,n-1})^2=2^{2n-4}.
\end{equation}
By~\eqref{CTRL_app_eq:P_lower_bound_stable} and~\eqref{CTRL_app_eq:P_lower_bound_stable_A} we finally get
\[
B_0'P_\mathrm{0}B_0+R_0=[P_{0}]_{n-1,n-1}+1\ge 2^{2n-4}+1
\]
\end{proof}

\begin{lemma}\label{CTRL_lem_app:stable_aux1}
We have:
\[
\snorm{\Delta'_1 P_0 (A_0+B_0K_{\star,0})}_2 \ge (0.5+o(1))(B_0'P_0B_0+R_0),
\]
where the $o(1)$ is in the large $n$ regime.
\end{lemma}
\begin{proof}
Let $e_i$ denote the $i$-th canonical vector in $\R^{n-1}$.
It is sufficient to show that
\[
\abs{(B_0'P_{0}B_0+R_0)^{-1}\Delta'_1P_{0}(A_0+B_0K_{\star,0})e_{n-1}} \ge 0.5+o(1).
\]
For simplicity we will denote:
\[
\alpha \triangleq [P_{0}]_{n-1,n-1},\quad\beta \triangleq [P_{0}]_{n-2,n-2},\quad \gamma\triangleq [P_{0}]_{n-1,n-2}.
\]
Due to the structure of $A_0$, we have 
\[A_0e_{n-1}=\rho e_{n-1}+2e_{n-2}.\]
Using this, we obtain
\begin{align}\label{CTRL_eq:aux_K_en}
K_{\star,0}e_{n-1}&=-(B_0'P_0B_0+1)^{-1}B_0'P_0A_0e_{n-1}=-(\alpha+1)^{-1}e'_{n-1}P_0(\rho e_{n-1}+2e_{n-2})\nonumber\\
&=-(\alpha+1)^{-1}(\rho \alpha+2\gamma).
\end{align}
Combining the above results
\begin{align*}
&(B_0'P_{0}B_0+R)^{-1}\Delta'_1P_{0}(A_0+B_0K_{\star,0})e_{n-1}=(B_0'P_0B_0+1)^{-1}e_{n-2}'P_0(A_0+B_0K_{\star,0})e_{n-1}\\
&=(\alpha+1)^{-1}\bigg\{e_{n-2}'P_0(\rho e_{n-1}+2e_{n-2})-e_{n-2}'P_0e_{n-1}(\alpha+1)^{-1}(\rho \alpha+2\gamma)\bigg\}\\
&=(\alpha+1)^{-1}\set{\rho \gamma+2\beta-\gamma(\alpha+1)^{-1}(\rho \alpha+2\gamma)}\\
&=2(\alpha+1)^{-1}\set{\beta-(\alpha+1)^{-1}\gamma^2}-(\alpha+1)^{-2}\rho \gamma \\
&\stackrel{i)}{=}\frac{2}{\alpha+1}\set{\beta-\frac{\gamma^2}{\alpha+1}}+o(1),
\end{align*}
where i) follows from Lemma~\ref{CTRL_lem_app:stable_aux2}.
What remains to show is that 
\begin{equation}\label{CTRL_eq:riccati_stable_aux}
    \frac{2}{\alpha+1}\set{\beta-\frac{\gamma^2}{\alpha+1}}=0.5+o(1).
\end{equation}
Using the algebraic Riccati equation:
\begin{align*}
\alpha&=e'_{n-1}A_0'P_0A_0e_{n-1}+1-e'_{n-1}A_0'P_0B_0(\alpha+1)^{-1}B_0'P_0A_0e_{n-1}\\
&=(\rho e_{n-1}+2e_{n-2})'P_0(\rho e_{n-1}+2e_{n-2})+1\\
&-(\rho e_{n-1}+2e_{n-2})'P_0e_{n-1}(\alpha+1)^{-1}e_{n-1}'P_0(\rho e_{n-1}+2e_{n-2})\\
&= \rho^2\alpha+4\beta+4\rho \gamma+1-\frac{(\rho \alpha+2\gamma)^2}{\alpha+1}\\
&=4\beta+\frac{\rho^2\alpha+4\rho \gamma+\alpha+1-4\gamma^2}{\alpha+1}.
\end{align*}
Dividing both sides with $\alpha+1$:
\begin{align*}
\frac{\alpha}{1+\alpha}=\frac{4}{\alpha+1}\set{\beta-\frac{\gamma^2}{\alpha+1}}+\frac{4\rho \gamma}{(\alpha+1)^2}+\frac{1+\rho^2 \alpha}{(1+\alpha)^2}
\end{align*}
Rearranging the terms gives:
\begin{align*}
\frac{2}{\alpha+1}\set{\beta-\frac{\gamma^2}{\alpha+1}}-0.5=-\frac{0.5}{1+\alpha}-\frac{2\rho \gamma}{(\alpha+1)^2}-\frac{1+\rho^2 \alpha}{2(1+\alpha)^2}
\end{align*}
By~Lemma~\ref{CTRL_lem_app:stable_aux2} the second term in the right-hand side is $o(1)$. By Lemma~\ref{CTRL_lem:exponential_riccati_stable}, $\alpha=\Omega(2^{2n})$, hence all remaining terms also go to zero, which completes the proof of~\eqref{CTRL_eq:riccati_stable_aux}.
\end{proof}

\begin{lemma}\label{CTRL_lem_app:stable_aux2}
Recall the notation in the proof of Lemma~\ref{CTRL_lem_app:stable_aux1} 
\[
\alpha\triangleq [P_{0}]_{n-1,n-1},\quad \gamma\triangleq [P_{0}]_{n-1,n-2}.
\]
Then, we have:
\[
\abs{\frac{\gamma}{(\alpha+1)^2}}=o(1)
\]
\end{lemma}
\begin{proof}
We use the relation:
\[
P_{0}=(A_0+B_0K_{\star,0})'P_{0}(A_0+B_0K_{\star,0})+Q_0+K_{\star,0}'R_0K_{\star,0}\succeq K_{\star,0}'R_0K_{\star,0}.
\]
Multiplying from the left and right by $e_{n-1}$ and by invoking~\eqref{CTRL_eq:aux_K_en} we obtain:
\[
\alpha\ge \paren{\frac{\rho \alpha+2\gamma}{\alpha+1}}^2=(\xi+\lambda)^2,
\]
where for simplicity we define $
\xi=\frac{\rho \alpha}{\alpha+1},\,\lambda=\frac{2\gamma}{\alpha+1}.
$
We can further lower bound the above expression by:
\[
\alpha \ge (\xi+\lambda)^2 \ge \xi^2+\lambda^2-2\xi\abs{\lambda}.
\]
This is a quadratic inequality and holds if and only if:
\[
\xi-\sqrt{\alpha}\le \abs{\lambda} \le \xi+\sqrt{\alpha}.
\]
As a result:
\[
2\frac{\abs{\gamma}}{\alpha+1}\le \rho+\sqrt{\alpha+1}
\]
which leads to
\[
\frac{\abs{\gamma}}{\alpha+1}\le 0.5\frac{\rho+\sqrt{\alpha+1}}{\alpha+1}=O(1/\sqrt{\alpha})=o(1)
\]
since $\alpha=\Omega(2^{2n})$.
\end{proof}